\documentclass[11pt]{article}
\usepackage{pdfpages}
\pdfoutput = 1

\usepackage{amsthm}
\usepackage{hyperref}
\usepackage{geometry}
\usepackage{url}
\usepackage{amssymb}
\usepackage[utf8]{inputenc} 
\usepackage{hyperref}       
\hypersetup{
	colorlinks   = true, 
	urlcolor     = blue, 
	linkcolor    = blue, 
	citecolor   = blue 
}
\usepackage{multicol}
\usepackage{url}            
\usepackage{booktabs}       
\usepackage{amsfonts}       
\usepackage{nicefrac}       
\usepackage{float}
\usepackage{xcolor}         
\usepackage{amsmath}
\allowdisplaybreaks
\usepackage{comment}
\usepackage{algorithm}
\usepackage{algorithmic}
\newtheorem{assumption}{\bf Assumption}
\newtheorem{lemma}{\bf Lemma}
\newtheorem{theorem}{\bf Theorem}
\newtheorem{definition}{\bf Definition}

\usepackage{thmtools,thm-restate}

\usepackage{amsmath}
\usepackage{graphicx}
\usepackage{subfigure}
\usepackage{ifthen}

\newcommand{\Pa}{\boldsymbol{\it Pa}}
\newcommand{\pa}{\boldsymbol{\it pa}}

\newcommand{\Sphere}{{\it Sphere}}

\DeclareMathOperator*{\argmin}{argmin}
\DeclareMathOperator*{\argmax}{argmax}

\newcommand{\E}{\mathbb{E}}
\newcommand{\I}{\mathbb{I}}

\newcommand{\bA}{\boldsymbol{A}}

\newcommand{\be}{\boldsymbol{e}}

\newcommand{\bL}{\boldsymbol{L}}

\newcommand{\bq}{\boldsymbol{q}}

\newcommand{\bS}{\boldsymbol{S}}
\newcommand{\bs}{\boldsymbol{s}}

\newcommand{\bU}{\boldsymbol{U}}
\newcommand{\bP}{\boldsymbol{P}}

\newcommand{\bV}{\boldsymbol{V}}
\newcommand{\bv}{\boldsymbol{v}}

\newcommand{\bM}{\boldsymbol{M}}
\newcommand{\bX}{\boldsymbol{X}}

\newcommand{\bx}{\boldsymbol{x}}

\newcommand{\bz}{\boldsymbol{z}}
\newcommand{\bZ}{\boldsymbol{Z}}

\newcommand{\btheta}{\boldsymbol{\theta}}

\newcommand{\compilehidecomments}{true}
\ifthenelse{ \equal{\compilehidecomments}{true} }{%
	\newcommand{\wei}[1]{}
	\newcommand{\nuoya}[1]{}
}{
	\newcommand{\wei}[1]{{\color{blue!50!black}  [\text{Wei:} #1]}}
	\newcommand{\nuoya}[1]{{\color{red!60!black} [\text{Nuoya:} #1]}}
}

\newcommand{\compilefullversion}{true}
\ifthenelse{\equal{\compilefullversion}{false}}{%
	\newcommand{\OnlyInFull}[1]{}
	\newcommand{\OnlyInShort}[1]{#1}
}{%
	\newcommand{\OnlyInFull}[1]{#1}%
	\newcommand{\OnlyInShort}[1]{}%
}%

\usepackage[utf8]{inputenc} 
\usepackage[T1]{fontenc}    
\usepackage{hyperref}       
\usepackage{url}            
\usepackage{booktabs}       
\usepackage{amsfonts}       
\usepackage{nicefrac}       
\usepackage{microtype}      
\usepackage{xcolor}         

\title{Combinatorial Pure Exploration of Causal Bandits}

\begin{document}
\title{\LARGE Combinatorial Pure Exploration of Causal Bandits}
\author{Nuoya Xiong\\
 IIIS, 
 Tsinghua University\\
 xiongny20@mails.tsinghua.edu.cn \\
 \and
            Wei Chen\\
Microsoft Research\\
 weic@microsoft.com
}
\date{}

\maketitle

\begin{abstract}
  The combinatorial pure exploration of causal bandits is the following online learning task: given a causal graph with unknown causal inference distributions, 
		in each round we choose a subset of variables to intervene or do no intervention,
		and observe the random outcomes of all random variables, with the goal that using as few rounds as possible, we can output an intervention that gives the best 
		(or almost best) expected outcome on the reward variable $Y$ with probability at least $1-\delta$, where $\delta$ is a given confidence level.
	We provide the first gap-dependent and fully adaptive pure exploration algorithms on two types of causal models --- the binary generalized linear model (BGLM) and general graphs.
	For BGLM, our algorithm is the first to be designed specifically for this setting and achieves polynomial sample complexity, while all existing algorithms for general graphs
	have either sample complexity exponential to the graph size or some unreasonable assumptions.
	For general graphs, our algorithm provides a significant improvement on sample complexity, and it nearly matches the lower bound we prove.
	Our algorithms achieve such improvement by a novel integration of prior causal bandit algorithms and prior adaptive pure exploration algorithms, 
	the former of which utilize the rich observational feedback in causal bandits but are not adaptive to reward gaps, while the latter of which have the issue in reverse.
\end{abstract}

\vspace{-3mm}
\section{Introduction}
\label{sec.intro}

\vspace{-2mm}
Stochastic multi-armed bandits (MAB) is a classical framework in sequential decision making \cite{robbins1952some,bubeck2012regret}. 
In each round, a learner selects one arm based on the reward feedback from the previous rounds, and receives a random reward of the selected arm sampled from an unknown distribution, 
	with the goal of accumulating as much rewards as possible over $T$ rounds.
This framework models the exploration-exploitation tradeoff in sequential decision making ---
	whether one should select the best arm so far based on the past observations or one should try some arms that have not been played much.
Pure exploration is an important variant of the multi-armed bandit problem, where the goal of the learner is not to accumulate reward but
	to identify the best arm through possibly adaptive explorations of arms.
Pure exploration of MAB typically corresponds to a testing phase where we do not need to pay penalty for exploration, and it has wide applications in 
	online recommendation, advertising, drug testing, etc.
%
%

Causal bandits, first introduced by \cite{lattimore2016causal}, integrates causal inference \cite{Pearl09} with multi-armed bandits.
In causal bandits, we have a causal graph structure $G=(\bX\cup \{Y\}\cup \bU,E)$, where $\bX\cup \{Y\}$ are observable causal variables with $Y$ being a special reward variable, $\bU$ are unobserved hidden variables, and $E$
	is the set of causal edges between pairs of variables.
For simplicity, we consider binary variables in this paper.
The arms are the interventions on variables $\bS\subseteq \bX$ together with the choice of null intervention (natural observation), i.e.
	the arm (action) set is $A \subseteq \{a = do(\bS=\bs) \mid \bS\subseteq \bX, \bs \in \{0,1\}^{|\bS|} \}$ with $do()\in A$, 
	where $do(\bS=\bs)$ is the standard notation for intervening the causal graph by setting $
	\bS$ to $\bs$~\cite{Pearl09}, and $do()$ means null intervention.
The reward of an action $a$ is the random outcome of $Y$ when we intervene with action $a$, and thus the expected reward is $\E[Y \mid a=do(\bS=\bs)]$.
In each round, one action in $A$ is played, and the random outcomes of all variables in $\bX\cup \{Y\}$ are observed.
Given the causal graph $G$, but without knowing its causal inference distributions among nodes, the task of combinatorial pure exploration (CPE) of causal bandits is to (adaptively) select 
	actions in each round,
	observe the feedback from all observable random variables, so that in the end the learner can identify the best or nearly best actions. Causal bandits are useful in many real scenarios. In drug testing, the physicians wants to adjust the dosage of some particular drugs to treat the patient. In policy design, the policy-makers select different actions to reduce the spread of disease. 
	
%
Existing studies on CPE of causal bandits either requires the knowledge of $P(\Pa(Y)\mid a)$ for all action $a$ or only consider causal graphs without hidden variables, and the algorithms proposed are not fully adaptive to reward gaps 
	\cite{lattimore2016causal,icml2018-propagatinginference}. 
In this paper, we study fully adaptive pure exploration algorithms and analyze their gap-dependent sample complexity bounds in the fixed-confidence setting.
More specifically, given a confidence bound $\delta \in (0,1)$ and an error bound $\varepsilon$, we aim at designing adaptive algorithms that output an action such that
	with probability at least $1-\delta$, the expected reward difference between the output and the optimal action is at most $\varepsilon$.
The algorithms should be fully adaptive in the follow two senses.
First, it should adapt to the reward gaps between suboptimal and optimal actions similar to existing adaptive pure exploration bandit algorithms, such that actions with larger gaps should be explored less.
Second, it should adapt to the observational data from causal bandit feedback, such that actions with enough observations already do not need further interventional rounds for exploration, similar to
	existing causal bandit algorithms.
We are able to integrate both types of adaptivity into one algorithmic framework, and with 
interaction between the two aspects, we achieve better adaptivity than either of them alone.

	
	
First we introduce a particular term named gap-dependent observation threshold, which is a non-trivial gap-dependent extension for a similar term in \cite{lattimore2016causal}. 
Then we provide two algorithms, one for the binary generalized linear model (BGLM) and one for the general model with hidden variables.
The sample complexity of both algorithms contains the gap-dependent observation threshold that we introduced, which shows significant improvement comparing to the prior work.
In particular, our algorithm for BGLM achieves a sample complexity polynomial to the graph size, while all prior algorithms for general graphs have exponential sample complexity;
	and our algorithm for general graphs match a lower bound we prove in the paper.
To our best knowledge, our paper is the first work considering a CPE algorithm specifically designed for BGLM, and the first work considering CPE on graphs with hidden variables, while
	all prior studies either assume no hidden variables or assume knowing distribution $P(\Pa(Y)\mid a)$ for parent of reward variable $\Pa(Y)$ and all action $a$, which is not a reasonable assumption.
%

To summarize, our contribution is to propose the first set of CPE algorithms on causal graphs with hidden variables
	and fully adaptive to both the reward gaps and the observational causal data. 
The algorithm on BGLM
	is the first to achieve a gap-dependent sample complexity polynomial to the graph size, while the algorithm for general graphs improves significantly on sample complexity and matches
	a lower bound.
Due to the space constraint, further materials including  experimental results, an algorithm for the fixed-budget setting, and all proofs are moved to the supplementary material.

{\bf Related Work.\ \ }
Causal bandit is proposed by~\cite{lattimore2016causal}, who discuss the simple regret for parallel graphs and general graphs with known probability distributions 
$P(\Pa(Y)\mid a)$. 
\cite{nair2020budgeted} extend algorithms on parallel graphs  to graphs without back-door paths, and \cite{maiti2021causal} extend the results to the general graphs. All of them are either regard $P(\Pa(Y)\mid a)$ as prior knowledge, or consider only atomic intervention. 
 The study by \cite{icml2018-propagatinginference} is the only one considering the general graphs with combinatorial action set, but their algorithm cannot work on causal graphs with hidden variables.
All the above pure exploration studies consider simple regret criteria that is not gap-dependent.
Cumulative regret is considered in \cite{nair2020budgeted,lu2020regret,maiti2021causal}.  
To our best knowledge, \cite{sen2017identifying} is the only one discussing gap-dependent bound for pure exploration of causal bandits for the fixed-budget setting, 
but it only considers the soft interventions (changing conditional distribution $P(X| \Pa(X))$) on {\em one single node}, which is different from causal bandits defined in \cite{lattimore2016causal}. 

Pure exploration of multi-armed bandit has been extensively studied in the fixed-confidence or fixed-budget setting \cite{2010colt_exploration,2014fixed_confidence,lilUCB,2006action_elimination,LUCB2012}.
PAC pure exploration is a generalized setting aiming to find the $\varepsilon$-optimal arm instead of exactly optimal arm \cite{2002_pac_pe,2004PAC,2014fixed_confidence}.  
In this paper, we utilize the adaptive LUCB algorithm in \cite{LUCB2012}.
CPE has also been studied for multi-armed bandits and linear bandits, etc.(\cite{CPE_full_partial},\cite{CPE},\cite{CPE2021}), but the feedback model in these studies either have feedback at the base arm level or have full or partial bandit feedback,
	which are all very different from the causal bandit feedback considered in this paper.


The binary generalized linear model (BGLM) is studied in \cite{li2017provably,CCB2022} for cumulative regret MAB problems. 
We borrow the maximum likelihood estimation method and its result in \cite{li2017provably,CCB2022} for our BGLM part, 
	but its integration with our adaptive sampling algorithm for the pure exploration setting is new.

\section{Model and Preliminaries}
\label{sec.model}
\vspace{-2mm}
{\bf Causal Models. \ \ }
From \cite{Pearl09}, a causal graph $G=(\bX \cup \{Y\}\cup \bU,E)$ is a directed acyclic graph (DAG) with a set of observed random variables $\bX \cup \{Y\}$ and a set of hidden random variables $\bU$, 
	where $\textbf{\textit{X}} = \{X_1,\cdots,X_n\}$, $\bU = \{U_1,\cdots,U_k\}$ are two set of variables and $Y$ is the special reward variable 
	without outgoing edges.  
In this paper, for simplicity, we only consider that $X_i$'s and $Y$ are binary random variables with support $\{0,1\}$. 
For any node $V$ in $G$, we denote the set of its parents in $G$ as $\Pa(V)$.
The set of values for $\Pa(X)$ is denoted by $\pa(X)$. 
The causal influence is represented by $P(V\mid \Pa(V))$, modeling the fact that the probability distribution of a node $V$'s value is determined by the value of its parents. 
Henceforth, when we refer to a causal graph, we mean both its graph structure $(\bX \cup \{Y\}\cup \bU,E)$ and 
its causal inference distributions $P(V \mid \Pa(V))$ for all $V\in \bX \cup \{Y\} \cup \bU$.
A parallel graph $G = (\bX\cup \{Y\}, E)$ is a special class of causal graphs with $\bX=\{X_1,\cdots,X_n\}$, $\bU=\emptyset$ and $E = \{X_1\to Y, X_2\to Y,\cdots,X_n\to Y\}$. 
An \emph{intervention} $do(\bS=\bs)$ in the causal graph $G$ means that we set the values of a set of nodes $\bS \subseteq \bX$ to $\bs$, 
while other nodes still follow the $P(V \mid \Pa(V))$ distributions.
An \emph{atomic intervention} means that $|\bS|=1$.
When $\bS = \emptyset$, $do(\bS=\bs)$ is the null intervention denoted as $do()$, which means we do not set any node to any value and just observe all nodes' values.  

In this paper, we also study a parameterized model with no hidden variables: \emph{the binary generalized linear model} (BGLM). 
Specifically, in BGLM, we have $\bU = \emptyset$, and $P(X=1\mid \Pa(X)=\pa(X))=f_X(\btheta_X\cdot \pa(X))+e_X$, where \textcolor{black}{$f_X$ is a strictly increasing function,
	$\btheta_X \in \mathbb{R}^{\Pa(X)}$ is the unknown parameter vector for $X$, $e_X$ is a zero-mean bounded noise variable that 
	guarantees the resulting probability to be within $[0,1]$}. 
\textcolor{black}{To represent the intrinsic randomness of node $X$ not caused by its parents, 
we denote $X_1=1$ as a global variable, which is a parent of all nodes.}

{\bf Combinatorial Pure Exploration of Causal Bandits.\ \ }
Combinatorial pure exploration (CPE) of causal bandits describes the following setting and the online learning task.
The causal graph structure is known but the distributions $P(V|\Pa(V))$'s are unknown.
The action (arm) space $\bA$ is a subset of possible interventions on combinatorial sets of variables, plus the null intervention,
	that is, $\bA \subseteq \{do(\bS=\bs) \mid \bS \subseteq \bX, \bs \in \{0,1\}^{|\bS|} \} $ and $\{do()\} \in \bA$.
For action $a = do(\bS=\bs)$, define $\mu_{a}=\mathbb{E}[Y\mid do(\bS=\bs)]$ to be the expected reward of action $do(\bS=\bs)$. 
Let $\mu^* = \max_{a\in \bA} \mu_a$.

%
In each round $t$, the learning agent plays one action $a\in \bA$, observes the sample values $\bX_t=(X_{t,1},X_{t,2}\cdots,X_{t,n})$ and $Y_t$
	of all observed variables.
The goal of the agent is to interact with the causal model with as small number of rounds as possible to
	find an action with the maximum expected reward $\mu^*$. 
More precisely, we mainly focus on the following PAC pure exploration with the gap-dependent bound in the {\em fixed-confidence setting}.
In this setting, we are given a confidence parameter $\delta \in (0,1)$ and an error parameter $\varepsilon \in [0,1)$, 
	and we want to adaptively play actions over rounds based on past observations, terminate at a certain round and output an action $a^o$ to guarantee that
	$\mu^*-\mu_{a^o}\le \varepsilon$ with probability at least $1-\delta$.
The metric for this setting is sample complexity, which is the number of rounds needed to output a proper action $a^o$.
Note that when $\varepsilon = 0$, the PAC setting is reduced to the classical pure exploration setting.
We also consider the {\em fixed budget setting} in the appendix, where given an exploration round budget $T$ and an error parameter $\varepsilon \in [0,1)$, 
	the agent is trying to adaptively play actions and
	output an action $a^o$ at the end of round $T$, so that the error probability $\Pr\{\mu_{a^o} < \mu^* - \varepsilon\}$ is as small as possible.


We study the gap-dependent bounds, meaning that the performance measure is related to the reward gap between the optimal and suboptimal actions, as defined below.
Let $a^*$ be one of the optimal arms. For each arm $a$, we define the gap of $a$ as
\begin{align}
    \Delta_a = \left\{
\begin{array}{lcc}
\mu_{a^*}-\max_{a \in \bA \setminus \{a^*\}}\{\mu_{a} \},      &      & {a = a^*}; \\
\mu_{a^*}-\mu_{a},     &      & {a\neq a^*}.\\
\end{array} \right.
\end{align}

We further sort the gaps $\Delta_a$'s for all arms and assume 
	$\Delta^{(1)}\le \Delta^{(2)}\cdots \le\Delta^{(|\bA|)}$, where $\Delta^{(1)}$ is also denoted as $\Delta_{\min}$.

\section{Gap-Dependent Observation Threshold}
\label{sec:obsthreshold}

In this section, we introduce the key concept of gap-dependent observation threshold, which is instrumental to the fix-confidence algorithms in the next two sections.
Intuitively, it describes for any action $a$ whether we can derive its reward from pure observations of the causal model.

We assume that $X_i$'s are binary random variables.
First, we describe terms $q_a \in [0,1]$ for each action $a \in \bA$, which can have different definitions in different settings. 
Intuitively, $q_a$ represents how easily the action $a$ is to be estimated by observation. 
For example, in \cite{lattimore2016causal}, for parallel graph with action set $\bA=\{do(X_i=x)\mid 1\le i\le n, x\in\{0,1\}\}\cup\{do()\}$, 
	for action $a = do(X_i =x)$, $q_a = P(X_i=x)$ represents the probability for action $do(X_i=x)$ to be observed,
	since in parallel graph we have $P(Y\mid X_i=x) = P(Y\mid do(X_i=x))$.
Thus, when $q_a=P(X_i=x)$ is larger, it is easier to estimate $P(Y\mid do(X_i=x))$ by observation.
We will instantiate $q_a$'s for BGLM and general graphs in later sections.
For $a=do()$, we always set $q_a=1$. 
Then, for set $q_a, a \in \bA$ we define the \emph{observation thershold} as follows:
\begin{definition}[Observation threshold \cite{lattimore2016causal}]
\label{def:othresholdm}
For a given causal graph $G$ and its associated $\{q_a\mid a \in \bA\}$, the {\em observation threshold} $m$ is defined as:
\begin{align} \label{eq:mdef}
	m = \min\{\tau \in [|\bA|]: |\{a \in \bA\mid q_a< 1/\tau\}|\le \tau\}.\end{align}
\end{definition} 

The observation threshold can be equivalently defined as follows:
When we sort $\{q_a\mid a \in \bA\}$ as $q^{(1)}\le q^{(2)}\le \cdots \le q^{|\bA|}$, $m=\min\{\tau: q^{(\tau+1)}\ge \frac{1}{\tau}\}$. 
Note that $m\le |\bA|$ always holds since $q_{do()} = 1$. 
In some cases, $m \ll |\bA|$.
For example, in parallel graph, when $P(X_i=1)=P(X_i=0)=\frac{1}{2}$ for all $i \in [n]$, $q_{do(X_i=1)}=q_{do(X_i=0)} = \frac{1}{2}$, $q_{do()}=1$.Then $m=2\ll 2n+1 = |\bA|$.
Intuitively, when we collect passive observation data without intervention, arms corresponding to $q^{(j)}$ with $j\le m$ are under observed while
	arms corresponding to $q^{(j)}$ with $j > m$ are sufficiently observed and can be estimated accurately.
Thus, for convenience we name $m$ as the observation threshold (the term is not given a name in \cite{lattimore2016causal}).

In this paper, we improve the definition of $m$ to make it gap-dependent, which would lead to a better adaptive pure exploration algorithm and sample complexity bound. Before introducing the definition, we first define the term $H_r$. Sort the arm set as $ q_{a_1}\cdot\max\{\Delta_{a_1},\varepsilon/2\}^2\le q_{a_2}\cdot \max\{\Delta_{a_2},\varepsilon/2\}^2\le \cdot \le q_{a_{|\bA|}}\cdot \max\{\Delta_{a_{|\bA|}},\varepsilon/2\}^2$, then $H_r$ is defined by 
\begin{align}
    H_r = \sum_{i=1}^r \frac{1}{\max\{\Delta_{a_i},\varepsilon/2\}^2}. \label{def:H}
\end{align}
\begin{definition}[Gap-dependent observation threshold]
\label{def:gapm}
For a given causal graph $G$ and its associated $q_a$'s and $\Delta_a$'s, 
	the {\em gap-dependent observation threshold}
	$m_{\varepsilon, \Delta}$ is defined as:
   \begin{align} \label{eq:gapmdef}
   	m_{\varepsilon, \Delta}  = \min\left\{\tau \in [|\bA|]: \left|\left\{a\in \bA \Bigg| q_a\cdot \max\left\{\Delta_a,\varepsilon/2\right\}^2<\frac{1}{H_{\tau}}\right\}\right|\le \tau\right\}.\end{align}
\end{definition}
\textcolor{black}{The Gap-dependent observation threshold can be regarded as a generalization of the observation threshold. 
Intuitively, when considering the gaps, $q_a\cdot \max\{\Delta_a, \varepsilon/2\}^2$ represents how easily the action $a$ would to be distinguished from the optimal arm.} 
To show the relationship between $m_{\varepsilon,\Delta}$ and $m$, we provide the following lemma. The proof of Lemma is in \OnlyInFull{
Appendix~\ref{sec:proof_lemma1}.
}
\OnlyInShort{the supplementary material.}

\begin{lemma}\label{lemma:relation_m_gap}
    $m_{\varepsilon,\Delta}\le 2m$.
\end{lemma}

Lemma \ref{lemma:relation_m_gap} shows that $m_{\varepsilon,\Delta}=O(m)$. 
In many real scenarios, $m_{\varepsilon,\Delta}$ might be much smaller than $m$. Consider some integer $n$ with $4 < n < |\bA|$, $\epsilon<1/n$,   $q_{a}=\frac{1}{n}$ for $a \in \bA \setminus \{do()\}$ and $q_{do()}=1$. 
Then $m=n$. Now we consider $\Delta_{a_1} = \Delta_{a_2} =\frac{1}{n}$, while other arms $a$ have $\Delta_a = \frac{1}{2}$.
Then $H_r \ge n^2$ for all $r\ge 1$. 
Then for $a\neq a_1,a_2$, we have  $q_a\cdot \max\{\Delta_a,\varepsilon/2\}^2\ge \frac{1}{4n}>\frac{1}{H_r}$, which implies that $m_{\varepsilon,\Delta}=2$. This lemma will be used to show that our result improves previous causal bandit algorithm  in \cite{lattimore2016causal}.



\vspace{-2mm}

\section{Combinatorial Pure Exploration for BGLM}
\label{sec.linear}

In this section, we discuss the pure exploration for BGLM, a general class of causal graphs with a linear number of parameters, as defined
	in Section~\ref{sec.model}. In this section, we assume $\bU = \emptyset$.
 Let $\btheta^*=(\btheta^*_X)_{X\in \bX\cup \{Y\}}$ be the vector of all weights. 
 Since $X_1$ is a global variable, 
 we only need to consider the action
 	set  $\bA \subseteq  \{do(\bS=\bs) \mid \bS \subseteq \bX\setminus\{X_1\}, \bs \in \{0,1\}^{|\bS|} \}$.  
Following \cite{li2017provably,CCB2022}, we have three assumptions:
\begin{assumption}
	\label{ass:derivativeupper}
	For any $X \in \bX\cup \{Y\}$, $f_X$ is twice differentiable. Its first and second order  derivatives can be upper bounded by constant $M^{(1)}$ and $M^{(2)}$.
\end{assumption}
\begin{assumption}
\label{ass:derivativelower}
    $\kappa :=\inf_{X \in \bX\cup \{Y\}, \bv \in [0,1]^{\Pa(X)}, ||\btheta-\btheta_X^*||\le 1}\dot{f}_X (\bv\cdot \btheta)>0$ is a positive constant.
\end{assumption}
\begin{assumption}
\label{ass:parentnontrivial}
    There exists a constant $\eta>0$ such that for any $X \in \bX\cup \{Y\}$ and $X' \in \Pa(X)$, for any $\bv \in \{0,1\}^{|\Pa(X)-2|}$ and $x \in \{0,1\}$, we have 
    \begin{gather}
        Pr[X'=x\mid \Pa(X)\setminus\{X',X_1\}=\bv]\ge \eta.
    \end{gather}
\end{assumption}
%

Assumptions \ref{ass:derivativeupper} and \ref{ass:derivativelower} are the classical assumptions in generalized linear model \cite{li2017provably}. 
Assumption \ref{ass:parentnontrivial} makes sure that each parent node of $X$ has some freedom to become 0 and 1 with a non-zero probability, even when the values of all other parents of $X$
	are fixed, and it is originally given in~\cite{CCB2022} with additional justifications.
Henceforth, we use $\sigma(\btheta,a)$ to denote the reward $\mu_a$ on parameter $\btheta$.
	
Our main algorithm, Causal Combinatorial Pure Exploration-BGLM (CCPE-BGLM), is given in Algorithm~\ref{alg:linear}.
The algorithm follows the LUCB framework~\cite{LUCB2012}, but has several innovations.
In each round $t$, we play three actions and thus it corresponds to three rounds in the general CPE model.
In each round $t$, we maintain $\hat{\mu}^t_{O,a}$ and $\hat{\mu}^t_{I,a}$ as the estimates of $\mu_{a}$ from the observational data and the interventional data, respectively.
For each estimate, we maintain its confidence interval, $[L_{O,a}^t, U_{O,a}^t]$ and $[L_{I,a}^t, U_{I,a}^t]$ respectively.

At the beginning of round $t$, similar to LUCB, we find two candidate actions, one with the highest empirical mean so far, $a_h^{t-1}$; and one with the highest UCB among the rest, $a_l^{t-1}$.
If the LCB of $a_h^{t-1}$ is higher than the UCB of $a_l^{t-1}$ with an $\varepsilon$ error, then the algorithm could stop and return $a_h^{t-1}$ as the best action.
However, the second stopping condition in line~\ref{line:stopcondition} is new, and it is used to guarantee that the observational estimates $\hat{\mu}^t_{O,a}$'s are from enough samples.
If the stopping condition is not met, we will do three steps.
The first step is the novel observation step comparing to LUCB.
In this step, we do the null intervention $do()$, collect observational data, use maximum-likelihood estimate adapted from \cite{li2017provably,CCB2022} to obtain parameter estimate 
	$\hat{\theta}_t$, and then use $\hat{\theta}_t$ to compute observational estimate $\hat{\mu}^t_{O,a} = \sigma(\hat{\theta}_t, a)$ for all action $a$, \textcolor{black}{where $\sigma(\hat{\theta}_t, a) $ means the reward for action $a$ on parameter $\hat{\theta}_t$.} This
	can be efficiently done by
	following the topological order of nodes in $G$ and parameter $\hat{\theta}_t$.
From $\hat{\mu}^t_{O,a}$, we obtain the confidence interval $[L_{O,a}^t, U_{O,a}^t]$ using the bonus term defined later in Eq.\eqref{eq:betaOI}.
In the second step, we play the two candidate actions $a_h^{t-1}$ and $a_l^{t-1}$ and update their interventional estimates and confidence intervals, as in LUCB.
In the third step, we merge the two estimates together and  set the final estimate $\hat{\mu}_a^t$ to be the mid point of the intersection of two confidence intervals.
While the second step follows the LUCB, the first and the third step are new, and they are crucial for utilizing the observational data to obtain quick estimates for many actions at once.

Utilizing observational data has been explored in past causal bandit studies, but they separate the exploration from observations and the interventions into two stages 
\cite{lattimore2016causal,nair2020budgeted}, and thus their algorithms are not adaptive and cannot provide gap-dependent sample complexity bounds.
Our algorithm innovation is in that we interleave the observation step and the intervention step naturally into the adaptive LUCB framework, so that we can achieve an adaptive balance between
	observation and intervention, achieving the best of both worlds.

To get the confidence bound for BGLM, we use the following lemma from~\cite{CCB2022}:
\begin{restatable}{lemma}{boundestimate}\label{lemma:bound_estimate}
	For an action $a=do(\bS=\bs)$ and any two weight vectors ${\btheta}$ and ${\btheta'}$, we have 
	\begin{align}\label{eq_bound_diff}
		|\sigma({\btheta},a)-\sigma(\btheta',a)|\le \E_{\be}\left[\sum_{X \in N_{\bS,Y}}|\bV_{X}^{\top}({\btheta}_{X}-\btheta'_{X})|M^{(1)}\right],
	\end{align}
	where $N_{\bS,Y}$ is the set of all nodes that lie on all possible paths from $X_1$ to $Y$ excluding $\bS$, $\bV_{X}$ is the value vector of 
	a sample of the parents of $X$ according to parameter $\btheta$, $M^{(1)}$ is defined in Assumption \ref{ass:derivativeupper}, and the expectation is taken
	on the randomness of the noise term $\be=(e_X)_{X \in \bX\cup \{Y\}}$ of causal model under parameter $\btheta$.
\end{restatable}

\begin{algorithm}[H]
	\caption{CCPE-BGLM$(G, \bA,\varepsilon,\delta, M^{(1)},M^{(2)}, \kappa, \eta, c)$}
	\label{alg:linear}
	\begin{algorithmic}[1]
	\STATE Input:causal general graph $G$, action set $\bA$, parameter $\varepsilon,\delta,M^{(1)},M^{(2)},\kappa, \eta,c$ in Assumptions \ref{ass:derivativeupper},\ref{ass:derivativelower}, \ref{ass:parentnontrivial} and \OnlyInFull{Lemma~\ref{lemma:lecue's}} \OnlyInShort{in Lemma 4 in supplementary material}.
	
	\STATE Initialize $M_{0,X}=I$ for all node $X$. $N_a=0, \hat{\mu}_a^0=0, L_a^0 = -\infty, U_a^0 = \infty$ for arms $ a\in \bA.$\;
	
	\FOR{$t=1,2,\cdots,$}
	
	    \STATE $a_h^{t-1}=\argmax_{a \in \bA} \hat{\mu}_a^{t-1}, a_l^{t-1}=\argmax_{a \in \bA\setminus\{ a_h^{t-1}\}}U_a^{t-1}$. \label{line:twocandidates}
	    
	    \IF{$U_{a_{l}^{t-1}}^{t-1}\le L_{a_h^{t-1}}^{t-1}+\varepsilon$ and $t\ge \max\{\frac{cD}{\eta^2}\log \frac{nt^2}{\delta}, \frac{1024(M^{(2)})^2(4D^2-3)D}{\kappa^4\eta}(D^2+\log\frac{3nt^2}{\delta})\}$}\label{line:stopcondition}
	    \RETURN $a_h^{t-1}$.
	    
	    \ENDIF 
	    
	    \STATE /* {\em Step 1. Conduct a passive observation and estimate from the observational data} */
	    
	    \STATE Perform action $do()$ and observe $\bX_t$ and $Y_t$. For $a=do()$, $N_a=N_a+1$.
	    
	    \STATE $\hat{\theta}_t=\mbox{BGLM-estimate}((\bX_1,Y_1),\cdots,(\bX_t, Y_t))$.
	    
	    \STATE For $a = do(\bS = \bs) \in \bA$, calculate $\hat{\mu}_{O,a}=\sigma(\hat{\theta}_t,\bS)$, and 
	    $[L_{O,a}^t, U_{O,a}^t] = [\hat{\mu}_{O,a}-\beta_{O}^a(t),\hat{\mu}_{O,a}+\beta_{O}^a(t)].$ /* $\beta_{O}^a(t)$ is defined in Eq.\eqref{eq:betaOI} */
	    
	    \STATE /* {\em Step 2. Do two interventions and estimate from the interventional data} */
	    
	    \STATE Perform actions $a_l^{t-1}$ and $a_h^{t-1}$, get the reward $Y_t^{(l)}$ and $Y_t^{(h)}$.
	    
	    \STATE $N_{a_l^{t-1}} = N_{a_l^{t-1}}+1, N_{a_h^{t-1}} = N_{a_h^{t-1}}+1. $
	    
	    \STATE For $a \in \{a_l^{t-1}, a_h^{t-1},do()\}$, update the empirical mean 
	    \STATE $\hat{\mu}_{I,a} =\sum_{j=1}^t\frac{1}{N_a}(\mathbb{I}\{a = a_l^{j-1}\}Y_j^{(l)}+\mathbb{I}\{a = a_h^{j-1}\}Y_j^{(h)}+ \mathbb{I}\{a = do()\}Y_j
	    )$ 
	    and 
	    $[L_{I,a}^t, U_{I,a}^t] = [\hat{\mu}_{I,a}-\beta_{I}(N_a),\hat{\mu}_{I,a}+\beta_{I}(N_a)]$. /* $\beta_{I}(t)$ is defined in Eq.\eqref{eq:betaOI} */\label{line:empiricalmean}
	    
	    \STATE /* {\em Step 3. Merge the observational estimate and the interventional estimate} */
	    
	    \STATE For $a \in \bA$, calculate $[L_a^t, U_a^t] = [L_{O,a}^t, U_{O,a}^t]\cap [L_{I,a}^t, U_{I,a}^t]$ and $\hat{\mu}_a^t=\frac{L_a^t+U_a^t}{2}$. \label{line:mergeestimate}

	
	
	\ENDFOR

	\end{algorithmic}
\end{algorithm}
\begin{multicols}{2}

The key idea in the design and analysis of the algorithm is to divide the actions into two sets --- the easy actions and the hard actions.
Intuitively, the easy actions are the ones that can be easily estimated by observational data, while the hard actions require direction playing these actions to get accurate estimates.
The quantity $q_a$ mentioned in Section~\ref{sec:obsthreshold} indicates how easy is action $a$, and it determines the 
	gap-dependent observational threshold $ m_{\varepsilon,\Delta}$ (Definition~\ref{def:gapm}), which essentially gives the number of hard actions.
In fact, the set of actions in Eq.\eqref{eq:gapmdef} with $\tau = m_{\varepsilon,\Delta}$ is the set of hard actions and the rest are easy actions. We need to define $q_a$ representing the hardness of estimation for each $a$.

\begin{algorithm}[H]
	\caption{BGLM-estimate}
	\label{alg:MLE}
	\begin{algorithmic}[1]
	\STATE Input: data pairs
	
	$((\bX_1, Y_1), (\bX_2, Y_2),\cdots, (\bX_t, Y_{t}))$
	\STATE Construct $(\bV_{t,X},X_t)$ for each $X$, where $\bV_{t,X}$ is the value of parent of $X$ at round $t$, $X_t$ is the value of $X$ at round t.
	
	\FOR {$X \in \bX\cup \{Y\}$}
	\STATE $M_{t,X}=M_{t-1,X}+\bV_{t,X}\bV_{t,X}^{\top}$, calculate $\hat{\btheta}_{t,X}$ by solving $\sum_{ i=1}^t (X_i-f_X(\bV_{i,X}^T\hat{\btheta}_{t,X}))\bV_{i,X}=0$.
	\ENDFOR
	\RETURN $\hat{\btheta}_t$.
	\end{algorithmic}
\end{algorithm}

\end{multicols}

For CCPE-BGLM, we define its $q_a^{(L)}$ as follows.
Let $D = \max_{X\in \bX\cup \{Y\}} |\Pa(X)|$. 
For node $\bS \subseteq \bX$, let $\ell_{\bS} = |N_{\bS, Y}|$.
Then  for $a=do(\bS=\bs)$, we define \begin{align}q_a^{(L)}=\frac{1}{\ell_{\bS}^2D^3}.\label{def:linear_qa}\end{align}
 Intuitively, based on Lemma~\ref{lemma:bound_estimate} and $\ell_{\bS} = |N_{\bS,Y}|$, a large $\ell_{\bS}$ means that 
 the right-hand side of Inequality~\eqref{eq_bound_diff} could be large, and thus 
  it is difficult to estimate $\mu_a$ accurately. 
  Hence the term $q_a^{(L)}$ represents how easy it is to estimate for action $a$. 
  Note that $q_a^{(L)}$ only depends on the graph structure and set $\bS$.
We can define $m^{(L)}$ and $m_{\varepsilon,\Delta}^{(L)}$ with respect to $q_a^{(L)}$'s by Definitions~\ref{def:othresholdm} and \ref{def:gapm}.
We use two confidence radius terms as follows, one from the estimate of the observational data, and the other from the estimate of the interventional data. 
\begin{gather} \label{eq:betaOI}
	\beta_{O}^{a}(t)=\frac{\alpha_O M^{(1)}D^{1.5}}{\kappa\sqrt{\eta}}\sqrt{\frac{1}{q_a^{(L)} t}\log \frac{3nt^2}{\delta}},
	\beta_{I}(t)=\alpha_I\sqrt{\frac{1}{t}\log \frac{|\bA|\log(2t)}{\delta}}.
\end{gather}
Parameters $\alpha_O$ and $\alpha_I$ are exploration parameters for our algorithm.  
For a theoretical guarantee, we choose $\alpha_O$ = $6\sqrt{2}$ and $\alpha_I$ = 2, but more aggressive $\alpha_O$ and $\alpha_I$ 
	could be used in experiments. (e.g. \cite{findingalleps-good}, \cite{in_practice_MAB}, \cite{lilUCB}) 
The sample complexity of CCPE-BGLM is summarized in the following theorem. 
\begin{theorem}\label{theorem:linear}
 With probability $1-\delta$, our {\rm CCPE-BGLM}$(G,\bA,\varepsilon,\delta/2)$ returns an $\varepsilon$-optimal arm with sample complexity
	\begin{align}
		T = O\left(H_{m_{\varepsilon,\Delta}^{(L)}}\log \frac{|\bA| H_{m_{\varepsilon,\Delta}^{(L)}}}{\delta}\right),
	\end{align}
	where $ m_{\varepsilon,\Delta}^{(L)}, H_{m_{\varepsilon,\Delta}^{(L)}}$ are defined in Definition~\ref{def:gapm} and Eq.\eqref{def:H} in terms of $q_a^{(L)}$'s for $a \in \bA\setminus \{do()\}$ defined in Eq.\eqref{def:linear_qa}.
\end{theorem}

If  we treat the problem as a naive $|\bA|$-arms bandit, the sample complexity of LUCB1 is $\widetilde{O}(H) = \widetilde{O}(\sum_{a \in \bA}\frac{1}{\max\{\Delta_a,\varepsilon/2\}^2})$, which may contain an exponential number of terms.  
Now note that $q_a^{(L)}\ge \frac{1}{n^5}$, it is easy to show that $m_{\varepsilon,\Delta}^{(L)}\le 2n^5$. Hence $H_{m_{\varepsilon,\Delta}^{(L)}}$ contains only a polynomial number of terms. 
Other causal bandit algorithms also suffer an exponential term, unless they rely on a strong and unreasonable assumption as describe in the related work.
We achieve an exponential speedup by (a) a specifically designed algorithm for the BGLM model, and (b) interleaving observation and intervention and making the algorithm fully adaptive.


The idea of the analysis is as follows. First, for the $ m_{\varepsilon,\Delta}$ hard actions, we rely on the adaptive LUCB to identify the best, and its sample complexity according to LUCB is
	$O(H_{m_{\varepsilon,\Delta}^{(L)}}\log (|\bA| H_{m_{\varepsilon,\Delta}^{(L)}}/\delta))$.
Next, for easy actions, we rely on the observational data to provide accurate estimates.
According to Eq.\eqref{eq:gapmdef}, every easy action $a$ has the property that $q_a\cdot \max\{\Delta_a,\varepsilon/2\}^2 \ge 1/ H_{m_{\varepsilon,\Delta}}$.
Using this property together with Lemma~\ref{lemma:bound_estimate}, we would show that the sample complexity for estimating easy action rewards is also 
	$O(H_{m_{\varepsilon,\Delta}^{(L)}}\log (|\bA| H_{m_{\varepsilon,\Delta}^{(L)}}/\delta))$.
Finally, the interleaving of observations and interventions keep the samply complexity in the same order.

\section{Combinatorial Pure Exploration for General Graphs}
\label{sec.general}

\subsection{CPE Algorithm for General Graphs}
In this section, we apply a similar idea to the general graph setting, which further allows the existence of hidden variables.
The first issue is how to estimate the causal effect (or the do effect) $\mathbb{E}[Y\mid do(\bS = \bs)]$ in general causal graphs from the observational data.
The general concept of identifiability~\cite{Pearl09} is difficult for sample complexity analysis.
Here we use the concept of {\em admissible sequence}~\cite{Pearl09} to achieve this estimation. 
\begin{definition}[Admissible sequence]
    An admissible sequence for general graph G with respect to $Y$ and $\bS=\{X_1,\cdots,X_k\} \subseteq \bX$ is a sequence of sets of variables 
    $\bZ_1,\cdots \bZ_k \subseteq \bX$ such that 
    
    (1) $\bZ_i$ consists of nondescendants of $\{X_i, X_{i+1},\cdots,X_k\}$, 
    
    (2) $(Y\perp \!\!\! \perp X_i\mid X_1,\cdots,X_{i-1}, \bZ_1,\cdots,\bZ_{i})_{G_{\underline{X}_i, \overline{X}_{i+1},\cdots,\overline{X}_{k}}}$
    , where $G_{\underline{X}}$ means graph $G$ without out-edges of $X$, and $G_{\overline{X}}$ means graph $G$ without in-edges of $X$.
\end{definition}
Then, for $\bS = \{X_1,\cdots,X_k\}$, $\bs=\{x_1,\cdots,x_k\}$, we can calculate $\mathbb{E}[Y\mid do(\bS=\bs)]$ by 
\begin{align}
 \mathbb{E}[Y\mid do(\bS = \bs)] &= \sum_{\bz}P(Y=1\mid \bS=\bs, \bZ_i=\bz_i,i\le k)\nonumber\\&\ \ \ \ \ \ \ \cdot P(\bZ_1=\bz_1)\cdots P(\bZ_k = \bz_k\mid \bZ_{i} = \bz_i, X_i=x_i, i\le k-1)\label{adm_seq_formula},
\end{align}
where $\bz$ means the value of $\cup_{i=1}^k \bZ_i$, and $\bz_i$ means the projection of $\bz$ on $\bZ_i$. 
For $a=do(\bS=\bs)$ with $|\bS|=k$, we use $\{\bZ_{a,i}\}_{i=1}^k$ to denote the admissible sequence with respect to $Y$ and $\bS$ , and $\bZ_a = \cup_{i=1}^k \bZ_{a,i}$.
$Z_a=|\bZ_a|$ and $Z = \max_a Z_a$. In this paper, we simplify $\bZ_{a,i}$ to $\bZ_i$ if there is no ambiguity.  

\textcolor{black}{For any $\bP\subseteq \bX$, denote $\bP_t = \bX_t|_{\bP}$ as the projection of $\bX_t$ on $\bP$. We define \begin{center}\begin{small}\begin{gather}T_{a,\bz}= \sum_{j=1}^t \mathbb{I}\{\bS_j=\bs, (\bZ_a)_j = \bz\},r_{a,\bz}(t)=\frac{1}{T_{a,\bz}}\sum_{j=1}^t \mathbb{I}\{\bS_j=\bs, (\bZ_a)_j=\bz\}Y_j\label{eq:rt}\\
n_{a,\bz,l}(t) =\sum_{j=1}^t \mathbb{I}\{ (\bZ_i)_j=\bz_i, (X_i)_j = x_i, i\le l-1\}\label{eq:nt}\\
p_{a,\bz,l}(t)=\frac{1}{n_{a,\bz,l}(t)}\sum_{j=1}^t \mathbb{I}\{(\bZ_l)_j = \bz_l, (\bZ_i)_j=\bz_i, (X_i)_j = x_i, i\le l-1\}\label{eq:pt}
\end{gather}\end{small}\end{center} where the $r_{a,\bz}(t)$ and $p_{a,\bz,l}(t)$ are the empirical mean of $P(Y\mid \bS=\bs, \bZ_a = \bz)$ and $P(\bZ_l=  \bz_l\mid \bZ_i=\bz_i, X_i=x_i, i\le l-1)$. }
Using the above Eq.\eqref{adm_seq_formula}, we estimate each term of the right-hand side for every $\bz \in \{0,1\}^{Z_a}$ to obtain an estimate for $\mathbb{E}[Y\mid a]$ \textcolor{black}{as follows:
\begin{small}\begin{align}
    \hat{\mu}_{O,a}=\sum_{\bz}r_{a,\bz}(t)\prod_{l=1}^k p_{a,\bz,l}(t).\label{eq:adm_seq_estimate}
\end{align}\end{small}}


\begin{algorithm}[t]
	\caption{{\rm CCPE-General}$(G,\bA,\varepsilon,\delta)$}
	\label{alg:algorithm3}
	\begin{algorithmic}[1]
	\STATE Input:causal graph $G$, action set $\bA$, parameter $\varepsilon,\delta$, admissible sequence $\{(Z_a)_i\}$ for each action $a \in \bA$
	
	\STATE Initialize $t=1$, $T_a=0, T_{a,\bz}=0, N_a = 0, \hat{\mu}_a=0$ for all arms $a \in \bA$, $\bz\in \{0,1\}^z$, $z\in [|X|]$.
	
	\FOR{$t=1,2,\cdots,$}
	
	\STATE $a^{t-1}_h = \argmax_{a \in \bA}\hat{\mu}_a^{t-1}, a_l^{t-1} = \argmax_{a \in \bA\setminus a_h^{t-1}}(U_a^{t-1})$
	
	\IF{$U_{a_l^{t-1}}\le L_{a_h^{t-1}}+\varepsilon$}
	\RETURN $a_h^{t-1}$
	\ENDIF
	
    \STATE \emph{/* Step 1. Conduct a passive observation and estimate from the observational data */}

	\STATE Perform $do()$ operation and observe $\bX_t$ and $Y_t$. For $a=do()$, $N_a=N_a+1$.
	
	
	\FOR{$a=do(\bS = \bs) \in \bA \setminus \{do()\}$  {\rm with an admissible sequence and} $\bS = \{X_1,\cdots,X_k\}, \bs=\{x_1,\cdots,x_k\}$}

	
	
	
	
	
	
	


	
	\STATE \textcolor{black}{Estimate $\hat{\mu}_{O,a}$ using \eqref{eq:adm_seq_estimate}
    and $[L_{O,a}^t, U_{O,a}^t] = [\hat{\mu}_{O,a}-\beta_{O}^a(T_a),\hat{\mu}_{O,a}+\beta_{O}^a(T_a)].$ \emph{/* $\beta_{O}^a(t)$ is defined in Eq.\eqref{eq:betaOIgeneral} */}}
	\ENDFOR

    \STATE \emph{/* Step 2. Do two interventions and estimate from the interventional data */}
	
%
%
%
    \STATE Perform actions $a_l^{t-1}$ and $a_h^{t-1}$, get the reward $Y_t^{(l)}$ and $Y_t^{(h)}$.

	\STATE $N_{a_l^{t-1}} = N_{a_l^{t-1}}+1, N_{a_h^{t-1}} = N_{a_h^{t-1}}+1. $

	\STATE For $a \in \{a_l^{t-1}$, $a_h^{t-1},do()\}$, update the empirical mean 
	$\hat{\mu}_{I,a}$ as Line~\ref{line:empiricalmean} in Algorithm~\ref{alg:linear}.
	
	\STATE Update
	$[L_{I,a}^t, U_{I,a}^t] = [\hat{\mu}_{I,a}-\beta_{I}(D_a),\hat{\mu}_{I,a}+\beta_{I}(D_a)]$. \emph{/* $\beta_{I}(t)$ is defined in Eq.\eqref{eq:betaOIgeneral} */}

    \STATE \emph{/* Step 3. Merge the observational estimate and the interventional estimate */}

%
	\STATE For $a \in \bA$, calculate $[L_a^t, U_a^t] = [L_{O,a}^t, U_{O,a}^t]\cap [L_{I,a}^t, U_{I,a}^t]$ and $\hat{\mu}_a^t=\frac{L_a^t+U_a^t}{2}$. \label{line:mergeestimate2}
\ENDFOR

\end{algorithmic}
\end{algorithm}

For general graphs,  there is no efficient algorithm to determine the existence of the admissible sequence and extract it when it exists.
But we could rely on several methods to find admissible sequences in some special cases.
First, we can find the {\em adjustment set}, a special case of admissible sequences. 
For a causal graph $G$, $\bZ$ is an adjustment set for variable $Y$ and set $\bS$ if and only if $P(Y=1\mid do(\bS=\bs)) = \sum_{\bz}P(Y=1\mid \bS=\bs, \bZ=\bz)P(\bZ=\bz)$.
There is an efficient algorithm for deciding the existence of a minimal adjustment set with respect to any set $\bS$ and $Y$ and finding it \cite{causal_adjustmentset}.
Second, for general graphs without hidden variables, the admissible sequence can be easily found by $\bZ_j = \Pa(X_j)\setminus (\bZ_1\cup \cdots \bZ_{j-1}\cup X_1\cdots \cup X_{j-1})$ (\OnlyInFull{See Theorem~\ref{theorem: adm_seq_nohidden} in Appendix~\ref{section: existence_adm_seq}}\OnlyInShort{Theorem 4 in the Appendix}). 
Finally, when the causal graph satisfies certain properties, there exist algorithms to decide and construct admissible sequences~\cite{existence_adm_seq}.

Algorithm~\ref{alg:algorithm3} provides the pseudocode of our algorithm CCPE-General, which has the same framework as Algorithm~\ref{alg:linear}.
The main difference is in the first step of updating observational estimates, in which we rely on the do-calculus formula~Eq.\eqref{adm_seq_formula}.

For an action $a = do(\bS = \bs)$ without an admissible sequence, define $q_a^{(G)}=0$, meaning that it is hard to be estimated through observation.
Otherwise, define $q_a$ as:
 \begin{align}
 	\label{eq:qaGeneral}
	q_a^{(G)} = \min_{\bz}\{q_{a,\bz}\}, \mbox{where } q_{a,\bz} = P(\bS=\bs, \bZ_a=\bz), \forall \bz \in \{0,1\}^{Z_a}.
\end{align}
 Similar to CCPE-BGLM, for $a = do(\bS=\bs)$ with $|\bS|=k$, we use observational and interventional confidence radius as:
\begin{align}
    \beta_{O}^a(t)=\alpha_{O}\sqrt{\frac{1}{t}\log \frac{20k|\bA|Z_aI_a\log (2t)}{\delta}}; 
    \beta_{I}(t)=\alpha_{I}\sqrt{\frac{1}{t}\log\frac{|\bA|\log(2t)}{\delta}}, \label{eq:betaOIgeneral}
\end{align}
where $\alpha_O$ and $\alpha_I$ are exploration parameters, and $I_a= 2^{Z_a}.$
For a theoretical guarantee, we will choose $\alpha_O$ = 8 and $\alpha_I$ = 2. 
Our sample complexity result is given below.
\begin{theorem}\label{theorem:general}
    With probability $1-\delta $, {\rm CCPE-General}$(G,\bA,\varepsilon,\delta/5)$ returns an $\varepsilon$-optimal arm with sample complexity 
\begin{align} \label{eq:theorem2}
    T = O\left(H_{m_{\varepsilon,\Delta}^{(G)}}\log \frac{|\bA| H_{m_{\varepsilon,\Delta}^{(G)}}}{\delta}\right),
\end{align}
where $m_{\varepsilon,\Delta}^{(G)}, H_{m_{\varepsilon,\Delta}^{(G)}}$ are defined in Definitions ~\ref{def:gapm} and ~\ref{def:H} in terms of $q_a^{(G)}$'s defined in Eq.\eqref{eq:qaGeneral}.
\end{theorem}



Comparing to LUCB1, since $m_{\varepsilon,\Delta}^{(G)}\le |\bA|$, our algorithm is always as good as LUCB1.
It is easy to construct cases where our algorithm would perform significantly better than LUCB1.
Comparing to other causal bandit algorithms, our algorithm also performs significantly better, especially when $ m_{\varepsilon,\Delta}^{(G)} \ll m^{(G)}$
	or the gap $\Delta_a$ is large relative to $\varepsilon$.
Some causal graphs with candidate action sets and valid admissible sequence are provided in \OnlyInFull{Appendix~\ref{appendix_A}}\OnlyInShort{the Appendix A}, and
more detailed discussion can be found in \OnlyInFull{Appendix~\ref{app:comparison}}\OnlyInShort{the Appendix}.

%


\subsection{Lower Bound for the General Graph Case}
\label{sec.lowerbound}

To show that our CCPE-General algorithm is nearly minimax optimal, we provide the following lower bound, which is based on parallel graphs.
We consider the following class of parallel bandit instance $\xi$ with causal graph $G = (\{X_1,\cdots,X_n,Y\}, E)$:
	the action set is $\bA = \{do(X_i=x)\mid x \in \{0,1\}, 1\le i\le n\}\cup\{do()\}$.
The $q^{(G)}_a$ in this case is reduced to $q^{(G)}_{do(X_i=x)}=P(X_i=x)$ and $q_{do()}=1$.
    Sort the action set as $ q^{(G)}_{a_1}\cdot\max\{\Delta_{a_1},\varepsilon/2\}^2\le q^{(G)}_{a_2}\cdot \max\{\Delta_{a_2},\varepsilon/2\}^2\le \cdots \le q^{(G)}_{a_{2n+1}}\cdot \max\{\Delta_{a_{2n+1}},\varepsilon/2\}^2$.
Let $p_{\min}=\min_{\bx \in \{0,1\}^n}P(Y=1\mid \bX=\bx), p_{\max} = \max_{\bx \in \{0,1\}^n}P(Y=1\mid \bX=\bx)$. 
Let $p_{\max}+2\Delta_{2n+1}+2\varepsilon\le0.9, p_{\min}+\Delta_{\min}\ge 0.1$. 
\begin{theorem}\label{theorem:lowerbound}
    For the parallel bandit instance class $\xi$ defined above, any $(\varepsilon,\delta)$-PAC algorithm has expected sample complexity at least
    \begin{equation}
      \Omega\left(\left(H_{m^{(G)}_{\varepsilon,\Delta}-1}-\frac{1}{\min_{i< m^{(G)}_{\varepsilon,\Delta}}\max\{\Delta_{a_i},\varepsilon/2\}^2}-\frac{1}{\max\{\Delta_{do(),\varepsilon/2}\}^2}\right)\log \left(\frac{1}{\delta}\right)\right).  
    \end{equation}
\end{theorem}
Theorem \ref{theorem:lowerbound} is the first gap-dependent lower bound for causal bandits, which needs brand-new construction and technique. 
Comparing to the upper bound in Theorem~\ref{theorem:general}, the main factor $H_{m^{(G)}_{\varepsilon,\Delta}}$ is
	the same, except that the lower bound subtracts several additive terms. 
The first term $H_{m_{\varepsilon,\Delta}-1}$ is almost equal to $H_{m_{\varepsilon,\Delta}}$ appearing in Eq.\eqref{eq:theorem2}, except the it omits the last and the smallest additive term in $H_{m_{\varepsilon,\Delta}}$.
The second term is to eliminate one term with minimal $\Delta_{a_i}$, which is common in multi-armed bandit. (\cite{2018Lattimore},\cite{eliminate_Delta_min_example}) 
The last term is because $do()$'s reward must be in-between $\mu_{do(X_i=0)}$ and $\mu_{do(X_i=1)}$ and thus cannot
	be the optimal arm.

\section{Future Work}
\label{sec.conclusion}



There are many interesting directions worth exploring in the future. First, how to improve the computational complexity for CPE of causal bandits is an important direction.
Second, one can consider developing efficient pure exploration algorithms for causal graphs with  partially unknown graph structures. 
Lastly, identifying the best intervention may be connected with the markov decision process and studying their interactions is also an interesting direction.



\OnlyInFull{
\appendix
\clearpage

\appendix
\onecolumn

\section*{Appendix}

\section{General Classes of Graphs Supporting Theorem 2}
\label{appendix_A}
For Theorem \ref{theorem:general}, in this section we show some graphs with small size of admissible sequence for all arms, which makes our result much better than previous algorithms. By comparison in the Appendix~\ref{app:comparison}, we show that if $2^{Z+l}\le |\bA|$, where $Z = \max_a Z_a, l = \max\{|\bS| \mid do(\bS=\bs) \in A\}$, our algorithm can perform better than previous classical bandit algorithms.

\paragraph{Two-layer graphs} Consider $\bX=A\cup B$, where $A=\{X_1,\cdots,X_k\}$ is the set of key variables, $B = \{X_{k+1}\cdots,X_n\}$ are the rest of variables. 
Now we consider $k\le \frac{1}{2}\log_2 n$, and the edge set is in $E\subseteq\{(X_i\to X_j)\mid X_i \in A, X_j \in A\}\cup\{(X_i\to X_j)\mid X_i \in A, X_j \in B\}\cup \{(X_i\to Y)\mid X_i \in B\}$.
There can also exist some hidden confounders between two variables in $A$, namely, $A_1\leftarrow U\to A_2$ for unobserved variables $U$ and $A_1,A_2 \in A$.
\begin{figure}[htb] 
	\centering 
	\includegraphics[width=0.50\textwidth]{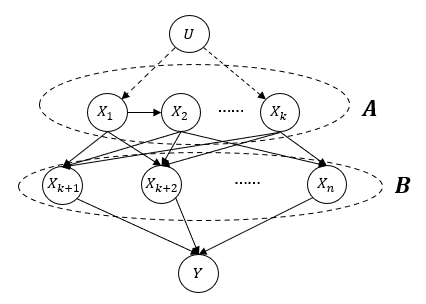}
	\caption{An Example of Two-layer Graphs}
	\label{Fig.1} 
\end{figure}

We define the action set as $\{do(\bS=\bs)\mid \bS\subset  B, |\bS|\le l, \bs \in \{0,1\}^{|\bS|}\}$ for some $l$. Then, since for arm $do(\bS=\bs)$, $A$ is the adjustment set for it, we know $Z_a\le k\le \frac{1}{2}\log_2 n$ for all action $a$. Then $2^{Z+l}\le \sqrt{n}\cdot 2^l<\binom{n}{l}\cdot 2^l< |\bA|$.

Consider the scenario in which a farmer wants to optimize the crop yield \cite{lattimore2016causal}. 
$A=\{X_1, X_2,\cdots, X_k\}$ are key elements influencing crop yields, such as temperature, humidity, and soil nutrient. $B=\{X_{k+1},\cdots,X_n\}$ are different kinds of crops,
	and $Y$ is the final total reward collected from all crops. 
Each kind of crop may be influenced by key elements in $A$ in different ways. 
Moreover, the elements in $A$ may have some causal relationships: higher humidity will lead to lower temperature. The above causal graph represents this problem very well.

\paragraph{Collaborative graphs}
Consider $\bX = \bX^1\cup \bX^2\cup \cdots \cup \bX^l$, where each $\bX^i (1\le i\le l)$ has at most $k\le \frac{1}{2}\log n$ nodes. The edge set is contained in $E=\{X \to Y\mid X \in \bX\}\cup \{X_i\to X_j\mid X_i,X_j \in \bX^t, 1\le t\le l\}$. In each subgraph $\bX^i$, we allow the existence of unobserved confounders between two variables in $\bX^i$. (We use dashed arrows to represent the confounders.)  We call this class of graphs collaborative graphs (see Figure~\ref{Fig.2}), 
	since it is modified by \cite{NEURIPS2021_collaborative} on collaborative causal discovery.
\begin{figure}[htbp] \label{fig:coll_graph}
	\centering 
	\includegraphics[width=0.50\textwidth]{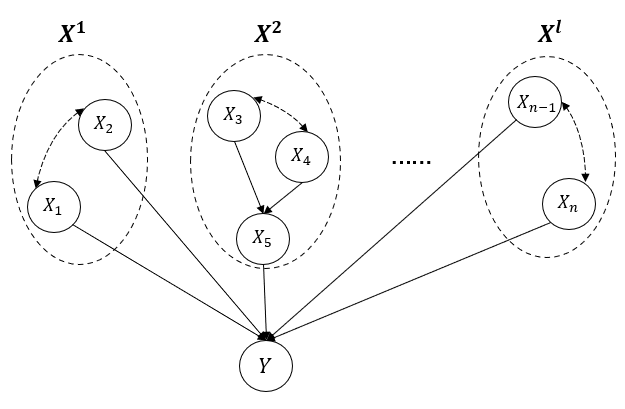}
	\caption{An Example of Collaborative Graphs}
	\label{Fig.2} 
\end{figure}

For simplicity, the action set is defined by $\{do(\bS = \bs)\mid  |\{\bS\cap \bX^i\}|\le 1, |\bS|\le d\}$. Then for a particular $\bS=\{X_{i_1},X_{i_2},\cdots,X_{i_d}\}$ and $i_1,\cdots,i_d$ such that $X_{i_j}\in \bX^{i_j}, 1\le j\le d$ for some $d\in [0,l]$
For these graphs, we know $T=\cup_{j=1}^d \bX^{i_j}\setminus \bS$ is a adjustment set (then also a admissible sequence) for $\bS$ and $Y$ with $|T|\le \frac{1}{2} d\log n$. Then $2^Z< n^{d/2}\cdot 2^d<\binom{n/k}{d}\cdot 2^d< |\bA|$ when $n$ is large. Collaborative graphs are useful 
	in many real-world scenarios. 
For example, many companies want to cooperate and maximize their profits. 
Then each subgraph $\bX^i(1\le i\le l)$ represents a company, and they want to find the best intervention to generate the maximum profit. 
\paragraph{Causal Tree}

Causal tree is a useful structure in real scenario, which is consider in \cite{NEURIPS2021_unknown_gr} and \cite{nips2019-causaltree}. In this class of graph, the underlying causal graph of causal model is a directed tree, in which all its edges point away from the root. Denote the root as layer 0, and layer $i$, $\bL_i$ contains all the nodes with distance $i$ to the root.  For simplicity, we assume all unobserved confounders point to two nodes in same layer. For a set $T$, its c-component $C_T$ means all the nodes connected to $T$ by only bi-directed edges (confounders) . 

For each action set $\{do(\bS=\bs)\mid \bS\subseteq \bX\}$,  we consider $\bS\cap \bL_i = \bS_i$. Then the sequence $\bZ_i = C_{\bS_i}\cup \Pa(C_{\bS_i})\setminus\{Z_0,\cdots,Z_{i-1}, S_0,\cdots,S_{i-1}\}$ is the admissible sequence. We give an example in Figure~\ref{Fig.3}. For example, if we consider action $do(\{X_3,X_4,X_8\}=\bs, \bs \in \{0,1\}^3)$, then the admissible sequence is $\bZ_1 = \{X_1,X_2\}, \bZ_2 = \emptyset, \bZ_3 = \{X_7\}$, and we can write 
\begin{align*}
    P(Y\mid do(X_3,X_4,X_8))=\sum_{X_1,X_2,X_7}&P(Y\mid X_1,X_2,X_3,X_4,X_7,X_8)\\&\ \ \ \ \ P(X_1,X_2)P(X_7\mid X_1,X_2,X_3,X_4).
\end{align*}

\begin{figure}[htbp] \label{fig:causal_tree}
	\centering 
	\includegraphics[width=0.60\textwidth]{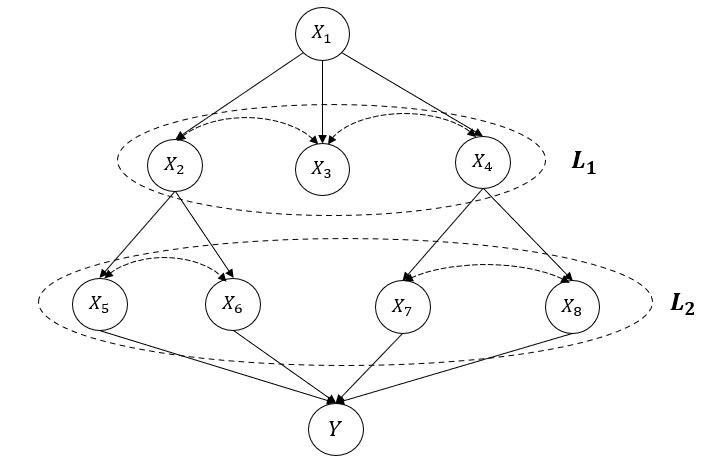}
	\caption{An Example of Causal Tree with Confounders}
	\label{Fig.3} 
\end{figure}

\section{More Detailed Comparison of Sample Complexity}
\label{app:comparison}

Here we provide a bit more detailed sample complexity comparison between our Theorem~\ref{theorem:general} on general graphs with hidden variables and prior studies.

\paragraph{Compare with LUCB1 algorithm}
Comparing to LUCB1, since $m_{\varepsilon,\Delta}^{(G)}\le |\bA|$, our algorithm will not perform worse than LUCB1. 
Our algorithm can also perform much better than LUCB1 algorithm in some cases. For example, when we consider $\bA = \{do(\bS=\bs)\mid |\bS|= k, \bs \in\{0,1\}^{|\bS|}\}$ for some constant $k$, we have $|\bA| = \binom{n}{k}\cdot 2^k$ . Assume $Z=\max_a Z_a \le c\log n$ for a constant $c$, and $q_a$ can be approximately $\Theta(\frac{1}{2^{Z_a+k}})$ ($q_a = \min_{\bz}P(\bS = \bs, \bZ_a = \bz)$) for all action $a$. Then we can get $m_{\varepsilon,\Delta}^{(G)} \approx 2^{Z+k}\le n^{c}\cdot 2^k  = o(|\bA|)$. Thus our algorithm performs much better than LUCB1.

\paragraph{Compare with previous causal bandit algorithms}Since there is no previous causal bandit algorithm working on combinatorial action set with hidden variables, 
we compare two previous causal bandit algorithms in some special cases. 
First, compare to \cite{lattimore2016causal} with parallel graph and atomic intervention, we first transfer the simple regret result in \cite{lattimore2016causal} to sample complexity $\widetilde{O}(\frac{m^{(G)}}{\varepsilon^2}\log(\frac{1}{\delta}))$. For parallel graph and $a = do(X_i=x)$, we know $q_a = P(X_i=x)$ since there is no parent for $X_i$, and our algorithm result is $\widetilde{O}(H_{m_{\varepsilon,\Delta}})$. Then  since $ m_{\varepsilon,\Delta}^{(G)}=O(m^{(G)})$ and $\max\{\Delta_a,\varepsilon/2\}\ge \varepsilon/2$, our algorithm always perform better. 
When the gap $\Delta_a$ is large relative to $\varepsilon$, our algorithm perform much better because of our gap-dependent sample complexity. 
\cite{icml2018-propagatinginference} consider combinatorial intervention on graphs without hidden variables, so we can compare our algorithm's result with theirs in this setting. 
We also transfer their simple regret result to sample complexity $\widetilde{O}(\frac{\max\{nC, n|\bA|\}}{\varepsilon^2}\log (1/\delta))$, where $C = \sum_{X \in \bX\cup\{Y\}}2^{|\Pa(X)|}$. 
Note that when $|\Pa(Y)|$ is large, $C\ge 2^{|\Pa(Y)|}$ can be really large. However, our algorithm even does not need the knowledge of $\Pa(Y)$.  Indeed,  considering $\max_{a= do(\bS = \bs)}|\bS| = k$ is a constant, and assume $Z\le \log C-k$ and $q_a = \Theta(\frac{1}{2^{Z+k}})$, we have $m_{\varepsilon,\Delta}^{(G)} \le \Theta(C)$, then our dominating term $H_{m_{\varepsilon,\Delta}}^{(G)}$ is smaller than $\frac{nC}{\varepsilon^2}$ because both $\max\{\Delta_a,\varepsilon/2\}\ge \varepsilon/2$ and $|M^{(G)}| = m_{\varepsilon,\Delta}^{(G)}\le nC$. Also, at the worst case our algorithm's sample complexity is not more than $\widetilde{O}(\frac{|\bA|}{\varepsilon^2}\log(\frac{1}{\delta}))$, while the algorithm in \cite{icml2018-propagatinginference} may result in $\widetilde{O}(\frac{n|\bA|}{\varepsilon^2}\log(1/\delta))$. The experiments are provided in Appendix~\ref{sec:experiment}.

In summary,  when compared to prior studies on causal bandit algorithms, our algorithm wins when the reward gaps are relatively large or the size of the admissible sequence is small; and
when compared to prior studies on adaptive pure exploration algorithms, our algorithm wins by estimating do effects using observational data and saving estimates on those easy actions.

\section{Proof of Theorems}

\subsection{Proof of Theorem \ref{theorem:linear}}\label{proofthm1}

\begin{proof}
We first provide a lemma in \cite{li2020online} to show the confidence for the maximum likelihood estimation.
	
	\begin{restatable}{lemma}{eigenvalueestimate}\label{lemma:eigenvalue_estimate}
		For one node $X \in \bX\cup \{Y\}$, assume Assumption \ref{ass:derivativeupper} and \ref{ass:derivativelower} holds, and 
		\begin{align}
			\lambda_{min}(M_{t,X})\ge \frac{512 D (M^{(2)})^2}{\kappa^4}(D^2+\ln\frac{3nt^2}{\delta}),\nonumber
		\end{align}
		with probability $1-\delta/nt^2$,  for any vector $\bv \in \mathbb{R}^{|\Pa(X)|}$, at all rounds $t$ the estimator $\hat{\btheta}_{t,X}$ in Algorithm \ref{alg:MLE} satisfy 
		\begin{align}
			|\bv^{\top}(\hat{\btheta}_{t,X}-\btheta_X^*)|\le \frac{3}{\kappa}\sqrt{\log (3nt^2/\delta)}||v||_{M_{t,X}^{-1}}.\nonumber
		\end{align}
	\end{restatable}
	
	Since we need to estimate $\btheta_{t,X}$ for all nodes, let $F_1$ be the event that the above inequality doesn't hold, then by union bound, $\Pr\{F_1\}\le n\sum_{t>1}\frac{\delta}{nt^2}\le \delta.$(We can consider $t>1$)
	Now from \cite{CCB2022}, the true mean $\sigma(\hat{\btheta}_t, X_i)$ and our estimation $\sigma(\btheta^*, X_i)$ can be bounded by Lemma \ref{lemma:bound_estimate}. We rewrite the Lemma \ref{lemma:bound_estimate} here, and give proof in Appendix~\ref{sec:proof_lemma_bound_estimate}.

	\boundestimate*
	By definition, for any action $a=do(\bM = s)$, $|P_{\bS,Y}|=\ell_a \in \{1,\cdots,n\}$. 
	We then introduce Lecu\'{e} and Mendelson's Inequality represented in \cite{nie2021matrix}. 
	\begin{lemma}[\cite{nie2021matrix} Lecu\'{e} and Mendelson's Inequality ]\label{lemma:lecue's}
		Let random column vector $\bv \in \mathbb{R}^D$, and $\bv_1,\cdots, \bv_n$ are $n$ independent copies of $\bv$. Assume $\bz \in \Sphere(D)$ such that 
		\begin{align}
			\Pr[|\bv^{\top} \bz|>\alpha^{1/2}]\ge \beta,\nonumber
		\end{align}
		then there exists a constant $c>0$ such that when $n\ge \frac{cD}{\beta^2}$
		\begin{align}
			\Pr\left[\lambda_{min}\left(\frac{1}{n}\sum_{i=1}^n \bv_i \bv_i^{\top} \right)\le \frac{\alpha\beta}{2}\right]\le e^{-n\beta^2/c}.\nonumber
		\end{align}
	\end{lemma}
	This lemma can help us to bound the minimum eigenvalue for $M_{t,X}=\sum_{1\le i\le t} \bV_{t,X}\bV_{t,X}^{\top}$. 
	To satisfy the condition for Lemma \ref{lemma:lecue's}, we provide a similar lemma in \cite{CCB2022}:

	\begin{lemma}\label{lemma_prob_lar}
		Under Assumption \ref{ass:parentnontrivial}, for any node $X \in \bX$ and $\bv \in Sphere(|\Pa(X)|)$,
		\begin{align}
			\Pr\left[|\Pa(X)\cdot \bz|> \frac{1}{\sqrt{4D^2-3}} \right]\ge \eta.\nonumber
		\end{align}
	\end{lemma}
	\begin{proof}
	    The proof is similar to \cite{CCB2022} with a modification. For completeness, we provide the full proof below. 
		Let $|\Pa(X)|=d\le D$, $\bz=(z_1,z_2,\cdots,z_d)$.
		Let $\Pa(X)=(X_{i_1}=X_1,X_{i_2},\cdots,X_{i_{d}})$ and $\pa(X)=(x_{i_1}=1,x_{i_2},\cdots,x_{i_{d}})$.
		We denote $d_0= \sqrt{d-1}+\frac{1}{2\sqrt{d-1}}$.
		If $|z_1|\ge \frac{d_0}{\sqrt{d_0^2+1}}$, then by Cauchy-Schwarz inequality, we can deduce that
		\begin{align*}
			|\pa(X)\cdot \bv|&\ge |z_1|-\sum_{i=2}^d|z_i|\\
			&\ge \frac{d_0}{\sqrt{d_0^2+1}}-\sqrt{(d-1)\sum_{i=2}^d |z_i|^2}\\
			&\ge \frac{d_0}{\sqrt{d_0^2+1}}-\sqrt{(d-1)(1-\frac{d_0^2}{d_0^2+1})}\\
			&=\frac{1}{2\sqrt{(d_0^2+1)(d-1)}}\\
			&=\frac{1}{4d^2-3}.
		\end{align*}
		Thus when $|z_1|\ge \frac{d_0}{\sqrt{d_0^2+1}}$, $|\Pa(X)\cdot \bz|>\frac{1}{4d^2-3}\ge \frac{1}{4D^2-3}$.
		If $|z_1|< \frac{d_0}{\sqrt{d_0^2+1}}$, assume $|z_2|=\max_{2\le i\le d}{|z_i|}$, then 
		\begin{equation}\label{boundofz_2}
			|z_2|\ge \frac{1}{\sqrt{d-1}}\sqrt{\sum_{i=2}^d|z_i|^2}\ge \frac{\sqrt{1-(d_0/\sqrt{d_0^2+1})^2}}{\sqrt{d-1}}=\frac{1}{\sqrt{4d^2-3}}.
		\end{equation}
		By Assumption \ref{ass:parentnontrivial}
		\begin{align*}
			&\ \ \ \ \ \ \Pr\{X_{i_1}=1, X_{i_2}=x_{i_2},\cdots,X_{i_d}=x_{i_d}\}\\&=\Pr\{X_{i_2}=x_{i_2}\mid X_{i_1}=1,X_{i_3}=x_{i_3},\cdots,X_{i_d}=x_{i_d}\}\cdot \Pr\{X_{i_1}=1,X_{i_3}=x_{i_3},\cdots,X_{i_d}=x_{i_d}\}\\
			&\ge \eta \cdot \Pr\{X_{i_1}=1,X_{i_3}=x_{i_3},\cdots,X_{i_d}=x_{i_d}\},
		\end{align*}
		we have 
		\begin{align*}
			&\ \ \ \ \ \Pr\left\{|\Pa(X)\cdot \bz|\ge \frac{1}{\sqrt{4D^2-3}}\right\}\\
			&=\sum_{x_{i_3},\cdots,x_{i_d}}\Pr\{X_{i_1}=1,X_{i_2}=1,X_{i_3}=x_{i_3}\cdots,X_{i_d}=x_{i_d}\}\cdot \mathbb{I}\left\{|(1,1,x_{i_3},\cdots,x_{i_d})\cdot (z_1,\cdots,z_d)|\ge \frac{1}{\sqrt{4D^2-3}}\right\}\\&
			+\sum_{x_{i_3},\cdots,x_{i_d}}\Pr\{X_{i_1}=1,X_{i_2}=0,X_{i_3}=x_{i_3}\cdots,X_{i_d}=x_{i_d}\}\cdot \mathbb{I}\left\{|(1,0,x_{i_3},\cdots,x_{i_d})\cdot (z_1,\cdots,z_d)|\ge \frac{1}{\sqrt{4D^2-3}}\right\}\\
			&\ge \eta\sum_{x_{i_3},\cdots,x_{i_d}}\Pr\{X_{i_1}=1,X_{i_3}=x_{i_3}\cdots,X_{i_d}=x_{i_d}\}\cdot \mathbb{I}\left\{|(1,1,x_{i_3},\cdots,x_{i_d})\cdot (z_1,\cdots,z_d)|\ge \frac{1}{\sqrt{4D^2-3}}\right\}\\
			&+\eta \sum_{x_{i_3},\cdots,x_{i_d}}\Pr\{X_{i_1}=1,X_{i_3}=x_{i_3}\cdots,X_{i_d}=x_{i_d}\}\cdot \mathbb{I}\left\{|(1,0,x_{i_3},\cdots,x_{i_d})\cdot (z_1,\cdots,z_d)|\ge \frac{1}{\sqrt{4D^2-3}}\right\}\\
			&\ge \eta \sum_{x_{i_3},\cdots,x_{i_d}}\Pr\{X_{i_1}=1,{i_3}=x_{i_3}\cdots,X_{i_d}=x_{i_d}\}\cdot\\& \left(\mathbb{I}\left\{|(1,1,x_{i_3},\cdots,x_{i_d})\cdot (z_1,\cdots,z_d)|\ge \frac{1}{\sqrt{4D^2-3}}\right\}+\mathbb{I}\left\{|(1,0,x_{i_3},\cdots,x_{i_d})\cdot (z_1,\cdots,z_d)|\ge \frac{1}{\sqrt{4D^2-3}}\right\}\right)\\
			&\ge \eta.
		\end{align*}
		where the last inequality is because \begin{equation*}
			\sum_{x_{i_3},\cdots,x_{i_d}}\Pr\{X_{i_1}=1,{i_3}=x_{i_3}\cdots,X_{i_d}=x_{i_d}\}=1,
		\end{equation*}
		and \begin{equation*}\label{equation_lemma_pg}
			\left(\mathbb{I}\left\{|(1,1,x_{i_3},\cdots,x_{i_d})\cdot (z_1,\cdots,z_d)|\ge \frac{1}{\sqrt{4D^2-3}}\right\}+\mathbb{I}\left\{|(1,0,x_{i_3},\cdots,x_{i_d})\cdot (z_1,\cdots,z_d)|\ge \frac{1}{\sqrt{4D^2-3}}\right\}\right)\ge 1.
		\end{equation*}
		The above equation is because otherwise
		\begin{align*}
			|z_2|=|(1,1,x_{i_3},\cdots,x_{i_d})\cdot (z_1,\cdots,z_d)-(1,0,x_{i_3},\cdots,x_{i_d})\cdot (z_1,\cdots,z_d)| <\frac{2}{\sqrt{4D^2-3}}\le \frac{2}{\sqrt{4d^2-3}},   
		\end{align*}
		which leads to a contradiction of Eq.~\eqref{boundofz_2}. We thus complete the proof of Lemma \ref{lemma_prob_lar}.
	\end{proof}
	Now let $F_2$ be the event
	\begin{align}
		F_2=\left\{\exists X \in \bX\cup\{Y\},\  \lambda_{min}\left(\frac{1}{t}\sum_{i=1}^t \bV_{i,X}\bV_{i,X}^{\top}\right)\le \frac{\eta}{2(4D^2-3)}, \  \forall t\ge \frac{cD}{\eta^2}\log \frac{nt^2}{\delta}\right\}.\nonumber
	\end{align}
	Then 
	\begin{align}
		\Pr\{F_2\}&\le n\sum_{t\ge (cD/\eta^2)\log (nt^2/\delta)}e^{-t\eta^2/c}\nonumber\\
		&\le n\sum_{t\ge (cD/\eta^2)\log (nt^2/\delta)}\frac{\delta}{nt^2}\nonumber\\
		&\le (\frac{\pi^2}{3}-1)\delta\nonumber\\&\le \delta.\nonumber
	\end{align}
	Now from Lemmas  \ref{lemma:bound_estimate}, \ref{lemma:eigenvalue_estimate} and \ref{lemma:lecue's}, for all $a = do(\bS=1)$, with probability $1-2\delta$, for all $t\ge \max\{\frac{cD}{\eta^2}\log \frac{nt^2}{\delta}, \frac{1024(M^{(2)})^2(4D^2-3)D}{\kappa^4\eta}(D^2+\ln\frac{1}{\delta})\}$, we can deduce that 
	\begin{align}
		\lambda_{min}(M_{t,X})\ge \frac{\eta t}{2(4D^2-3)}\nonumber.
	\end{align}
	Then
	\begin{align}
		|\sigma(\hat{\btheta}_t,\bS)-\mu_a| &\le \sum_{X \in P_{\bS,Y}}|\bV_{t,X}^{\top}(\hat{\btheta}_t-\btheta^*)|M^{(1)}\nonumber\\
		&\le \frac{3M^{(1)}}{\kappa}\sqrt{\log (3nt^2/\delta)}\sum_{X \in P_{\bS,Y}}||\bV_{t,X}||_{M_{t,X}^{-1}}\nonumber\\
		&\le  \frac{3M^{(1)}}{\kappa}\sqrt{\log (3nt^2/\delta)} \sum_{X \in P_{\bS,Y}}\frac{\sqrt{D}}{\sqrt{\lambda_{min}(M_{t,X})}}\nonumber\\
		&\le \frac{3\sqrt{2}M^{(1)}}{\kappa}{\sqrt{D(4D^2-3)}}\sqrt{\log(3nt^2/\delta)}\sum_{X' \in P_{\bS,Y}}\frac{1}{\sqrt{\eta t}}\nonumber\\
		&\le \frac{6\sqrt{2}M^{(1)}}{\kappa\sqrt{\eta}}\sqrt{\frac{\ell_a^2D^3}{t}\log(3nt^2/\delta)}\nonumber\\
		& = \frac{6\sqrt{2}M^{(1)}}{\kappa\sqrt{\eta}}\sqrt{\frac{1}{q_at}\log(3nt^2/\delta)}\nonumber\\
		&=\beta_{O}^a(t).\nonumber
	\end{align}
	Now we prove that Algorithm \ref{alg:linear} must terminate after $\lceil T\rceil$ rounds, where $T = \frac{1152(M^{(1)})^2}{\kappa^2\eta}H_{m_{\varepsilon,\Delta}^{(L)}}\log \frac{3nT^2}{\delta}+16H_{m_{\varepsilon,\Delta}^{(L)}}\log \frac{4|\bA|\log(2T)}{\delta}$. In the following proof, we assume $F_1$ and $F_2$ do not happen. Then the true mean will not out of observational confidence bound and interventional confidence bound.

	When $t\ge T_1$ such that $T_1 =  \frac{1152(M^{(1)})^2}{\kappa^2\eta}H_{m_{\varepsilon,\Delta}^{(L)}} \log \frac{3nT_1^2}{\delta}$, 
	for all $a\neq do()$ such that $q_{a}^{(L)}\ge \frac{1}{H_{m_{\varepsilon,\Delta}^{(L)}}\cdot \max\{\Delta_{a},  \varepsilon/2\}^2}$, let $\beta_a(t)=\frac{U_a^t-L_a^t}{2}$, we have 
	\begin{align}
		\beta_a(t) := \frac{U_a^t-L_a^t}{2}\le \beta_{O}^a(\lceil T_1 \rceil)\le \frac{6\sqrt{2}M^{(1)}}{\kappa\sqrt{\eta}}\sqrt{\frac{1}{q_a^{(L)}\lceil T_1\rceil}\log (3nt^2/\delta)}\le \frac{\max\{\Delta_a,\varepsilon/2\}}{4}.\nonumber
	\end{align}

	Then we provide the following lemma:
	
	\begin{lemma}\label{lemma:twosmall}
		    If at round $t$, we have 
		    
		    $$\beta_{a_h^t}(t)\le \frac{\max\{\Delta_{a_h^t},\varepsilon/2\}}{4},\beta_{a_l^t}(t)\le \frac{\max\{\Delta_{a_l^t},\varepsilon/2\}}{4}, $$
        where $a_h^t, a_l^t$ are the actions performed by algorithm at round $t$. then the algorithm will stop at round $t+1$.	
		\end{lemma}
		\begin{proof}
		    From above, if the optimal arm $a^*= a_h^t$,
		    \begin{align}
		        \hat{\mu}_{a_l^{t}}+\beta_{a_l^t}(t)&\le \mu_{a_l^t}+2\beta_{a_l^t}(t)\nonumber\\
		        &\le \mu_{a_l^t}+\frac{\max\{\Delta_{a_l^t},\varepsilon/2\}}{2}\nonumber\\
		        &\le \mu_{a_h^t}-\Delta_{a_l^t}+\frac{\max\{\Delta_{a_l^t},\varepsilon/2\}}{2}\nonumber\\
		        &\le \hat{\mu}_{a_h^t}+\beta_{a^*}(T_{a^*}(t))-\Delta_{a_l^t}+\frac{\max\{\Delta_{a_l^t},\varepsilon/2\}}{2}\nonumber\\
	           &\le \hat{\mu}_{a_h^t}-\beta_{a^*}(T_{a^*}(t))+\frac{\max\{\Delta_{a^*},\varepsilon/2\}+\max\{\Delta_{a_l^t},\varepsilon/2\}}{2}-\Delta_{a_l^t}\nonumber\\
	           &\le \hat{\mu}_{a_h^t}-\beta_{a^*}(T_{a^*}(t))+\frac{\Delta_{a^*}+\varepsilon/2+\Delta_{a_l^t}+\varepsilon/2}{2}-\Delta_{a_l^t}\nonumber\\
	           &\le \hat{\mu}_{a_h^t}-\beta_{a^*}(T_{a^*}(t))+\varepsilon.\nonumber
		    \end{align}
		    If optimal arm $a^*\neq a_h^t$, and the algorithm doesn't stop at round $t+1$, then we prove $a^*\neq a_l^t$. Otherwise, assume $a^*=a_l^t$
		    \begin{align}
		        \hat{\mu}_{a_h^t}^t &\le \mu_{a_h^t}^t+\frac{\max\{\Delta_{a_h^t},\varepsilon/2\}}{4}\\&= \mu_{a_l^t}^t-\Delta_{a_h^t}+\frac{\max\{\Delta_{a_h^t},\varepsilon/2\}}{4}\\&\le
		        \mu_{a_l^t}^t-\frac{3\Delta_{a_h^t}}{4}+\varepsilon/4\\
		        &\le \hat{\mu}_{a_l^t}^t+\frac{\max\{\Delta_{a^*},\varepsilon/2\}}{4}-\frac{3\Delta_{a_h^t}}{4}+\varepsilon/4\\&\le
		        \hat{\mu}_{a_l^t}^t+\varepsilon/2-\frac{\Delta_{a_h^t}}{2}.
		    \end{align}
		    From the definition of $a_{h}^t$, we know $\varepsilon>\Delta_{a_h^t}\ge\Delta_{a^*}, \beta_{a_h^t}(t)\le \varepsilon/4, \beta_{a_l^t}(t)\le \varepsilon/4$.
		    Then $\hat{\mu}_{a_l^t}^t+\beta_{a_l^t}(t)+\beta_{a_h^t}(t)\le \hat{\mu}_{a_l^t}+\varepsilon/2\le \hat{\mu}_{a_h^{t}}^t+\varepsilon,$ which means the algorithm stops at round $t+1$.
		    
		    Now we can assume $a^*\neq a_l^t, a^*\neq a_h^t$. Then 
		    \begin{align}
		          \mu_{a_l^t}+2\beta_{a_l^t}(t)\ge\hat{\mu}_{a_l^t}+\beta_{a_l^t}(t)\ge \hat{\mu}_{a^*}+\beta_{a^*}(T_{a^*}(t))\ge \mu_{a^*}=\mu_{a_l^t}+\Delta_{a_l^t}.
		    \end{align}
		    Thus 
		    \begin{align}
		        \Delta_{a_l^t}\le 2\beta_{a_l^t}(t)\le \frac{\max\{\Delta_{a_l^t},\varepsilon/2\}}{2},
		    \end{align}
		    which leads to $\Delta_{a_l^t}\le \varepsilon/2, \beta_{a_l^t}(t)\le \varepsilon/8$. Since 
		    
		    Also, 
		    \begin{align}
		        \mu_{a_h^t}+\beta_{a_h^t}(t)\ge\hat{\mu}_{a_h^t}\ge \hat{\mu}_{a_l^t}\ge \mu_{a^*}-\beta_{a_l^t}(t)= \mu_{a_h^t}+\Delta_{a_h^t}-\beta_{a_l^t}(t),
		    \end{align}
		    which leads to 
		    \begin{align}
		        \frac{\max\{\Delta_{a_h^t},\varepsilon/2\}}{4}\ge \Delta_{a_h^t}-\varepsilon/8,
		    \end{align}
		    and $\Delta_{a_h^t}\le \varepsilon/2, \beta_{a_h^t}(t)\le \varepsilon/8.$ Hence
		    $\hat{\mu}_{a_l^t}^t+\beta_{a_l^t}(t)+\beta_{a_h^t}(t)\le \hat{\mu}_{a_l^t}+\varepsilon/2\le \hat{\mu}_{a_h^{t}}^t+\varepsilon,$ which means the algorithm stops at round $t+1$.
		\end{proof}

	Denote $N_a(t)$ as the value of variable $N_a$ at round $t$. So by Lemma~\ref{lemma:twosmall}, when $t \ge T_1$, at each round at least one intervention will be performed on some actions $a$ with $\beta_a(t)\ge \frac{\max\{\Delta_a,\varepsilon/2\}}{4}$, which implies that $q_a<\frac{1}{H_{m_{\varepsilon,\Delta}^{(L)}}\cdot \max\{\Delta_a,\varepsilon/2\}^2}$, and $N_a(t) \le \frac{64}{\max\{\Delta_a,\varepsilon/2\}^2}\log \frac{|\bA|\log(2t)}{\delta}$ (Since $\beta_a(t)\le\beta_I^a(t) = 2\sqrt{\frac{1}{t}\log\frac{|\bA|\log(2t)}{\delta}}$). Denote the set of these arms as $M$, so we have 
	
	\begin{align*}
	    T-T_1&\le \sum_{a \in M}\frac{64}{\max\{\Delta_a,\varepsilon/2\}^2}\log \frac{|\bA|\log(2t)}{\delta}\\
	    &\le 64(H_{m_{\varepsilon,\Delta}^{(L)}})\log \frac{|\bA|\log(2t)}{\delta},
	\end{align*}
	Hence 
	\begin{align*}
	T\le \frac{1152(M^{(1)})^2}{\kappa^2\eta}H_{m_{\varepsilon,\Delta}^{(L)}}\log \frac{3nT^2}{\delta}+64H_{m_{\varepsilon,\Delta}^{(L)}}\log \frac{|\bA|\log(2T)}{\delta}.
	\end{align*}
	Now we prove a sample complexity bound for Algorithm \ref{alg:linear} by the lemma above:
	\begin{lemma}
		If $T=NQ\log \frac{3nT^2}{\delta}+64Q\log \frac{|\bA|\log(2T)}{\delta}$ for some constant $N$, then $T=O(Q\log (Q|\bA|/\delta))$.\label{lemma:transfer_format}\end{lemma}
		\begin{proof}
		
		Since $f(x) = x-NQ\log\frac{3nx^2}{\delta}-64Q\log \frac{|\bA|\log (2x)}{\delta}$ is a increasing function when $T\ge 64Q$, we only need to show that 
		there exists a constant $C\ge 64$ such that $f(CQ\log \frac{Q|\bA|}{\delta})\ge f(T) = 0$. 
			Then 
			\begin{align}
				f(CQ\log \frac{Q|\bA|}{\delta}) &=  CQ\log\frac{Q|\bA|}{\delta}- NQ\log\frac{3n}{\delta}-2NQ\log(\frac{CQ\log(Q|\bA|/\delta)}{\delta})\nonumber\\&-64Q\log \frac{|\bA|\log(2(CQ\log (Q|\bA|/\delta)))}{\delta}\nonumber\\
			    &\ge (C-2N\log C-N)Q\log\frac{Q|\bA|}{\delta} - 64Q\log \frac{|\bA|}{\delta}\nonumber\\&\ \ -(64+2N)Q\log (\log 2C + \log (Q\log Q|\bA|/\delta))\nonumber\\
			    &\ge (C-2N\log C -N-64)Q\log \frac{Q|\bA|}{\delta}  \nonumber\\ &\ \ -(64+2N)Q\log(\log 2CQ) - (64+2N)Q\log(\log Q|\bA|/\delta)\label{eq:logxy_le_logx+y}\\&\ge 
			    (C-2N\log C-N-192-(64+2N)\log\log 2C)Q\log(Q|\bA|/\delta).\label{eq:final_result_ineq}
			\end{align}
			The equation \eqref{eq:logxy_le_logx+y} and \eqref{eq:final_result_ineq} are based on $\log(x+y)\le \log(xy)= \log x + \log y$ when $x,y\ge 2$. 
			Then choose $C$ such that  $C-2N\log C-N-192-(64+2N)\log\log 2C\ge 0$, so we complete the proof.
		\end{proof}
		
		Hence, by Lemma~\ref{lemma:transfer_format} with $N = \frac{1152(M^{(1)})^2}{\kappa^2\eta}$, we know the total sample complexity is 
		\begin{align*}
		    T = O(H_{m_{\varepsilon,\Delta}^{(L)}}\log \frac{H_{m_{\varepsilon,\Delta}^{(L)}}|\bA|}{\delta}).
		\end{align*}
		
		Finally, we prove the correctness of our algorithm. 
	Since the stopping rule is $\hat{\mu}_{a_l^{t}}^{t}+\beta_{a_l^t}(t)\le \hat{\mu}_{a_h^{t}}^{t}-\beta_{a_h^t}(t)+\varepsilon$,
	if $a^*\neq a_h^t$, we have 
	\begin{align}
	    \mu_{a_h^t}+\varepsilon\ge \hat{\mu}_{a_h^{t}}-\beta_{a_h^t}(t)+\varepsilon&\ge \hat{\mu}_{a_l^{t}}+\beta_{a_l^t}(t)
	    \\&\ge \hat{\mu}_{a^*}+\beta_{a^*}(T_{a^*}(t))
	    \\&\ge \mu_{a^*}.
	\end{align}
	Hence either $a^*=a_h^t$ or $a_h^t$ is $\varepsilon$-optimal arm. Thus, we complete the proof.

\end{proof}

\subsection{Proof of Theorem \ref{theorem:general}}
\begin{proof}
    In this proof, we denote $T_{a,\bz}(t), T_{a}(t), N_a(t)$ are the value of $T_{a,\bz}, T_a, N_a$ respectively. 
	For conveniece, we prove {\rm CCPE-General}$(G,\varepsilon,\delta)$ outputs a $\varepsilon$-optimal arm with probability $1-3\delta$. For simplity, we denote $H_{m_{\varepsilon,\Delta}^{(G)}}$ as $H^{(G)}$. 
	In round t, $T_{a,\bz}(t)=\sum_{j=1}^t \mathbb{I}\{X_{j,i}=x, \Pa(X)_j = \bz\}, \hat{q}_{a,\bz} = \frac{T_{a,\bz}(t)}{t}$. By Chernoff bound, at round $t$ such that $q_{a,\bz}(t) \ge \frac{6}{t}\log \frac{6|\bA|I_a}{\delta}$, with probability at most $\delta/3|\bA|I_a$,
	\begin{align}
		|\hat{q}_{a,\bz}-q_{a,\bz}| > \sqrt{\frac{6q_{a,\bz}}{t}\log \frac{6|\bA|I_a}{\delta}}.\nonumber
	\end{align}
	Hence 
	\begin{gather}\label{eq_28}
		\hat{q}_{a}=\min_{\bz}\{\hat{q}_{a,\bz}\}\le \min_{\bz}\{q_{a,\bz}+\sqrt{\frac{6q_{a,\bz}}{t}\log \frac{6|\bA|I_a}{\delta}}\}= q_{a}+\sqrt{\frac{6q_a}{t}\log \frac{6|\bA|I_a}{\delta}}.
	\end{gather}
	When $q_a\ge \frac{3}{t}\log \frac{6|\bA|I_a}{\delta}$,
	$f(x)=x-\sqrt{\frac{6x}{t}\log \frac{6|\bA|I_a}{\delta}}$ is a increasing function.
	\begin{gather}\label{eq_29}
		\hat{q}_{a}\ge \min_{\bz}\{q_{a,\bz}-\sqrt{\frac{6q_{a,\bz}}{t}\log \frac{6|\bA|I_a}{\delta}}\} =  q_{a}-\sqrt{\frac{6q_a}{t}\log \frac{6|\bA|I_a}{\delta}}.
	\end{gather}
	Let $F_1$ be the event that at least one of above inequalities doesn't hold, then $\Pr\{F_1\}\le \delta$. Now let $F_2$ and $F_3$ be the event that during some round t, when $t$ is large the true mean of an arm is out of range $[L_{O,a}^t, U_{O,a}^t]$ and $[L_{I,a}^t, U_{I,a}^t]$  respectively. Following anytime confidence bound,$\Pr\{F_3\}\le \delta$. By   Lemma~\ref{lemma:chernoff-anytimeconfidencebound} and \ref{lemma:general_obs_confidencebound} we prove   $\Pr\{F_2\}\le 3\delta$.
	
	To prove the concentration bound, we need the following lemma, which is a Chernoff-type anytime confidence bound for Bernoulli variables. To our best knowledge, it is the first anytime confidence bound based on Chernoff inequality. 
	
	\begin{lemma}\label{lemma:chernoff-anytimeconfidencebound}
	    For $X_1,X_2,\cdots,X_n$ drawn from Bernoulli distribution with mean $\mu$, denote $\bar{X} = \sum_{i=1}^n X_i$, then for all round $t$ 
	    we have 
	    $$P(\overline{X}-\mu>2\sqrt{\frac{3\mu}{t}\log \frac{20\log(2t)}{\delta}}, \forall t\ge \frac{3}{\mu}\log\frac{20\log(2t)}{\delta})\le 1-\delta.$$
	\end{lemma}
	The main proof is achieved  by modification on part of Lemma 1 in \cite{lilUCB}. For completeness, we provide the full proof here.
	Let $S_t = \sum_{i=1}^t (X_i-\mu)$ and $\phi(x) = \sqrt{3\mu x\log(\frac{\log(x)}{\delta})}$. We define the sequence $\{u_i\}_{i\ge 1}$ as follows: $u_0 = 1, u_{k+1} = \lceil (1+C)u_k\rceil$, where $C$ is a constant. Then for simple union bound and Chernoff inequality, we have 
	\begin{align*}
	    P(\exists k\ge 1: S_{u_k}\ge \sqrt{1+C}\phi(u_k))&\le \sum_{k=1}^{\infty}\exp\left\{-\frac{(1+C)\cdot 3\mu u_k \log(\frac{\log(u_k)}{\delta})}{3\mu\cdot u_k}\right\} \\ &\le \exp\left\{-(1+C)\log\left(\frac{\log(u_k)}{\delta}\right)\right\}\\&
	    \le \sum_{k=1}^{\infty} \left(\frac{\delta}{k\log (1+C)}\right)^{1+C}\\&
	    \le \left(1+\frac{1}{C}\right)\log \left(\frac{\delta}{\log(1+C)}\right)^{1+C}.
	\end{align*}
	
    Then we proof Chernoff-type maximal Inequality:
    \begin{align}\label{eq:chernoff-maximal}
        P(\exists t  \in [n], S_t\ge x)\le \exp\left\{-\frac{x^2}{3\mu n}\right\}.
    \end{align}
    
    First, we know $\{S_t\}$ is a martingale and then $\{e^{S_t}\}$ is a non-negative submartingale. By Doob's submartingale inequality, we have 
    
    \begin{align*}P(\sup_{0\le i\le n} S_i\ge x) = P(\sup_{0\le i\le n}e^{S_i}\ge e^{sx})\le \frac{\mathbb{E}[e^{s\cdot S_n}]}{e^{sx}} &= \frac{(\mu e^{s \cdot (1-\mu)}+(1-\mu)e^{-s\mu})^n}{e^{sx}} \\&= \frac{((1-\mu)+\mu e^s)^n}{e^{sx+sn\mu}}\\
    &\le \frac{e^{n\mu\cdot (e^s-1)}}{e^{sx+sn\mu}}.
    \end{align*}
    Choose $s = \ln(1+\frac{x}{n\mu})$, by the proof of Chernoff bound with $\mu\ge \frac{3}{t}\log\frac{20\log(2t)}{\delta}$, we can easily get 
    
    \begin{align*}
        P(\sup_{0\le i\le n} S_i\ge x)\le \exp\left\{\frac{-x^2}{3\mu n}\right\}.
    \end{align*}
    
    Now with this inequality, we can derive the lemma. 
    \begin{align*}
        &\ \ \ \ \ P(\exists t \in \{u_k+1, \cdots, u_{k+1}-1\}: S_t-S_{u_k}\ge \sqrt{C}\phi(u_{k+1})) \\ &\le P(\exists t\in [u_{k+1}-u_k-1]: S_t\ge \sqrt{C}\phi(u_{k+1}))\\ &\le \exp\left\{-C\cdot \frac{u_{k+1}}{u_{k+1}-u_k-1}\log \left(\frac{\log(u_{k+1})}{\delta}\right)\right\}\\
        &\le \exp\left\{-(1+C)\log \left(\frac{\log(u_{k+1})}{\delta}\right)\right\}\\
        &\le \left(\frac{\delta}{(k+1)\log (1+C)}\right)^{1+C}\\
        &\le \left(\frac{\delta}{\log (1+C)}\right)^{1+C}.
    \end{align*}
    
    Now with probability at least $1-(2+1/C) \left(\frac{\delta}{\log(1+C)}\right)^{1+C}$, for $u_k\le t\le u_{k+1}$, we have 
    \begin{align*}
        S_t &= S_t-S_{u_k}+S_{u_k}\\&\le \sqrt{C}\phi(u_{k+1})+\sqrt{1+C}\phi(u_k)\\
        &\le (1+\sqrt{C})\phi((1+C)t).
    \end{align*}
    
    Now denote $\delta'=(2+1/C)\left(\frac{\delta}{\log(1+C)}\right)^{1+C}$, $\delta = \log(1+C) \left(\frac{C\delta'}{2+C}\right)^{\frac{1}{1+C}}$, we have with probability $1-\delta'$
    \begin{align*}
        P\left(S_t\ge (1+\sqrt{C})\sqrt{3(1+C)\mu t\log\left(\left(\frac{2+C}{C\delta'}\right)^{\frac{1}{1+C}}\cdot \frac{\log(1+C)t}{\log(1+C)}\right)}\right)\le 1-\delta'.
    \end{align*}
    
    Choose $C=0.25$, and note that $\frac{\log(1.25t)}{\log(1.25)}\le \frac{\log(2t)}{\log 2}$, $(2.25/0.25)^{0.8}/\log (2) < 10$ and $1.5*\sqrt{1.25}<2$, we complete the lemma's proof.
	
	\begin{lemma}\label{lemma:T_a,z,l(t)bound}
	    Denote $T_{a,\bz,l}(t)$ is the number of observations  from round 1 to round $t$ in which $\bZ_{a,i}=\bz_i, X_i=x_i ,i\le l-1$. Then we have $T_{a,\bz,l}(t)\ge 2^{k-l+1+|\bZ_{a,l}|}T_{a,\bz}(t)$.
	\end{lemma}
	\begin{proof}
	    The proof is straightforward. Since $T_{a,\bz}(t)$ is the number of observations from round 1 to round $t$ in which $\bZ_i = \bz_i, X_i = x_i, 1\le i\le k.$ Hence the number of observations for $\bZ_i = \bz_i, X_i = x_i$ for $i\le l-1$ is at least $2^{|\bZ_l|} \cdot 2^{k-(l-1)}\cdot T_{a,\bz}(t) = 2^{k-l+1+|\bZ_l|}T_{a,\bz}(t)$  
	\end{proof}
	
	\begin{lemma}\label{lemma:general_obs_confidencebound}
		With probability $1-3\delta$, for all round $t$,
		\begin{align}
			|\hat{\mu}_{obs,a}-\mu_a|<8\sqrt{\frac{1}{T_a(t)}\log\frac{20kZ_aI_a|\bA|\log (2t)}{\delta}},\label{eq:keyobs_confidence}
		\end{align}
		where $I_a = 2^{|\bZ_a|}.$
	\end{lemma}
	\begin{proof}
	    If $T_a(t)\le 12\log\frac{20kZ_aI_a|\bA|\log (2t)}{\delta}$, then the right term of \eqref{eq:keyobs_confidence} is greater than 1, and this lemma always holds.
	    In this proof, we denote $\bZ_{a,i}$ as $\bZ_i$ for simplity.
	    By classical anytime confidence bound, we know with probability $1-\delta/(k\cdot I_a)$, for all round $t$ we have

		\begin{align*}
			|r_{a,\bz}(t)-P(Y=1\mid X=x, \bZ_a=\bz)|&\le \sqrt{\frac{4}{T_{a,\bz}(t)}\log\frac{|\bA|\log (2t)}{\delta}}.
		\end{align*}
		First, let $s(t) = 20kZ_aI_a|\bA|\log(2t)$, if $t<\frac{6}{q_a}\log(s(t)/\delta)$, then let $Q=\frac{6}{q_a}\log(s(t)/\delta)$, based on $T_a(t)\ge 12\log\frac{s(t)}{\delta}$, then 
		\begin{align*}
		    P\left(t<\frac{6}{q_a}\log(1/\delta)\right)\le P\left(T_a(Q)\ge 12\log\frac{s(t)}{\delta}\right). 
		\end{align*}
		Thus by Chernoff bound, we know 
		\begin{align*}
		    P\left(T_a(Q)\ge 12\log\frac{s(t)}{\delta}\right)= P\left(\hat{q}_a(Q)\ge 2q_a\right)\le \delta,
		\end{align*}
		where $\hat{q}_a(Q) = \frac{T_a(Q)}{Q}$.
		
        Hence with probability at least $1-\delta$, now we have $t\ge \frac{6}{q_a}\log(s(t)/\delta)$.
		Also, since $\hat{P}(\bZ_i=z_i,X_i=x_i, i\le l-1)=T_{a,\bz,l}(t)/t$, by Chernoff bound, when $t\ge \frac{6}{q_a}\log (s(t)/\delta)$, with probability
		$1-\exp\{-\frac{P(\bZ_i=z_i,X_i=x_i, i\le l-1)\cdot t}{3}\}\ge 1-\delta$, we have
		\begin{align*}
		    \hat{P}(\bZ_i=z_i,X_i=x_i, i\le l-1)&\le 2P(\bZ_i=z_i,X_i=x_i, i\le l-1).
		\end{align*}
	    
		 Now by Lemma~\ref{lemma:chernoff-anytimeconfidencebound} and Lemma~\ref{lemma:T_a,z,l(t)bound}, with probability $1-\delta/(k\cdot I_a)$, since
		 
		 \begin{align*}
		     P(\bZ_l = \bz_l\mid \bZ_i = \bz_i, X_i = x_i, i\le l-1)&\ge \frac{ q_a}{P(\bZ_i=z_i, X_i=x_i, i\le l-1)}\\
		     &\ge \frac{ q_a}{2\hat{P}(\bZ_i=z_i, X_i=x_i, i\le l-1)}\\
		     &\ge \frac{q_at}{2T_{a,\bz,l}(t)}\\
		     &\ge \frac{3}{T_{a,\bz,l}(t)}\log \frac{20kI_a|\bA|\log(2t)}{\delta}.
		 \end{align*}
		 
		 By Lemma~\ref{lemma:chernoff-anytimeconfidencebound} we have 
		 \begin{align*}
		     &\ \ \ \ |p_{a,\bz,l}(t)-P(\bZ_l = \bz_l\mid \bZ_i = \bz_i, X_i = x_i, i\le l-1)|\\
		     &\le 2\sqrt{\frac{3P(\bZ_l = \bz_l\mid \bZ_i = \bz_i, X_i = x_i, i\le l-1)}{T_{a,\bz,l}(t)}\log \frac{20kI_a|\bA|\log(2t)}{\delta}}\\
		     &\le 2\sqrt{\frac{3P(\bZ_l = \bz_l\mid \bZ_i = \bz_i, X_i = x_i, i\le l-1)}{2^{k-l+|\bZ_l|+1}T_{a,\bz}(t)}\log \frac{20kZ_aI_a|\bA|\log(2t)}{\delta}}.
		 \end{align*}
		 Thus by union bound, with probability $1-\delta$, we have 
		\begin{align}
		    &\ \ \ \ \sum_{\bz_l}(p_{a,\bz,l}(t)-P(\bZ_l = \bz_k\mid \bZ_i = \bz_i, X_i = x_i, i\le l-1)) \nonumber \\ &\le 
		  \sum_{\bz:p_{a,\bz,l}(t)\ge P(\bZ_l = \bz_l\mid \bZ_i = \bz_i, X_i = x_i, i\le l-1)}(p_{a,\bz,l}(t)-P(\bZ_l = \bz_l\mid \bZ_i = \bz_i, X_i = x_i, i\le l-1))\nonumber  \\& = \frac{1}{2}\sum_{\bz}|p_{a,\bz,l}(t)-P(\bZ_l = \bz_l\mid \bZ_i = \bz_i, X_i = x_i, i\le l-1)|\label{eq:halfprobability} \\&\le \sum_{\bz}\sqrt{\frac{3P(\bZ_l = \bz_l\mid \bZ_i = \bz_i, X_i = x_i, i\le l-1)}{2^{k-l+|\bZ_l|+1}T_{a,\bz}(t)}\log \frac{20kZ_aI_a|\bA|\log(2t)}{\delta}}.\nonumber 
		\end{align}
		The equation \eqref{eq:halfprobability} is because $\sum_{\bz_k}p_{a,\bz,k}(t) = \sum_{\bz}P(\bZ_k=\bz_k\mid \bZ_i=\bz_i, X_i=x_i, i\le k-1)= 1$.
		Now we denote $$\hat{P}_{a,\bz,l}(t) =p_{a,\bz,1}(t)\cdots p_{a,\bz,k}(t),$$
		
		$$P_{a,\bz,l} = P(\bZ_1=\bz_1)\cdots P(\bZ_l = \bz_l\mid \bZ_i = \bz_i, X_i= x_i, i\le l-1).$$
		
		Hence we get
		\begin{align}
			&\ \ \ \ \ \hat{\mu}_{obs,a}\\&= \sum_{\bz}r_{a,\bz}(t)\cdot \hat{P}_{a,\bz,k}(t)\nonumber\\
			&\le \sum_{\bz}r_{a,\bz}(t)\cdot \hat{P}_{a,\bz,k-1}(t)\cdot P_t(\bZ_k=\bz_k\mid \bZ_i=\bz_i, X_{i}=x_i, i\le k-1)\nonumber\\&\ \ \ \ \ +\sum_{\bz}r_{a,\bz}(t)\hat{P}_{a,\bz,k-1}(t)\sqrt{\frac{3P(\bZ_k = \bz_k\mid \bZ_i=\bz_i, X_{i}=x_i, i\le k-1)}{2^{|\bZ_k|+1}T_{a,\bz}(t)}\log \frac{20kZ_aI_a|\bA|\log(2t)}{\delta}}\nonumber\\
			& \le \sum_{\bz}r_{a,\bz}(t)\cdot \hat{P}_{a,\bz,k-1}(t)\cdot P_t(\bZ_k=\bz_k\mid \bZ_i=\bz_i, X_{i}=x_i, i\le k-1)\nonumber\\&\ \ \ \ +\sum_{\bz}\hat{P}_{a,\bz,k-1}(t)\cdot \sqrt{\frac{3P(\bZ_k = \bz_k\mid \bZ_i=\bz_i, X_{i}=x_i, i\le k-1))}{2^{|\bZ_k|+1}T_{a,\bz}(t)}\log \frac{20kZ_aI_a|\bA|\log(2t)}{\delta}}\nonumber\\
			&\le \sum_{\bz}r_{a,\bz}(t)\cdot \hat{P}_{a,\bz,k-1}(t)\cdot P_t(\bZ_k=\bz_k\mid \bZ_i=\bz_i, X_{i}=x_i, i\le k-1)\nonumber\\&\ \ \ \ +\sum_{\bz_k}\cdot \sqrt{\frac{3P(\bZ_k = \bz_k\mid \bZ_i=\bz_i, X_{i}=x_i, i\le k-1))}{2^{|\bZ_k|+1}T_{a,\bz}(t)}\log \frac{20kZ_aI_a|\bA|\log(2t)}{\delta}}\nonumber\\
			&\le \sum_{\bz}r_{a,\bz}(t) \hat{P}_{a,\bz,k-1}(t)\cdot P_t(\bZ_k=\bz_k\mid \bZ_i=\bz_i, X_{i}=x_i, i\le k-1)\nonumber\\
			&\ \ \ \ + \sqrt{\frac{3\cdot 2^{|\bZ_k|}}{2\cdot 2^{|\bZ_k|}T_{a,\bz}(t)}\log \frac{20kZ_aI_a|\bA|\log(2t)}{\delta}}  \ \ \ \ \ (\mbox{Cauchy-Schwarz Inequality})\nonumber\\
			&\le \sum_{\bz}r_{a,\bz}(t)\cdot \hat{P}_{a,\bz,k-1}(t)\cdot P_t(\bZ_k=\bz_k\mid \bZ_i=\bz_i, X_{i}=x_i, i\le k-1)\\&\ \ \ \ \ + \sqrt{\frac{3}{2\cdot T_{a,\bz}(t)}\log \frac{20kZ_aI_a|\bA|\log(2t)}{\delta}}\nonumber\\
			&\le \cdots \nonumber\\
			&\le \sum_{\bz}r_{a,\bz}(t)P_{a,\bz,k}+\sum_{i=1}^k \sqrt{\frac{3}{2^{i}T_{a,\bz}(t)}\log \frac{20kZ_aI_a|\bA|\log(2t)}{\delta}}\nonumber\\
			&\le \mu_a+\frac{\sqrt{2}}{\sqrt{2}-1}\sqrt{\frac{3}{T_{a,\bz}(t)}\log \frac{20kZ_aI_a|\bA|\log(2t)}{\delta}} + \sqrt{\frac{4}{T_{a,\bz}(t)}\log\frac{\log(2t)}{\delta}}.\nonumber\\
			&\le \mu_a+ 8\sqrt{\frac{1}{T_a(t)}\log\frac{20kZ_aI_a|\bA|\log(2t)}{\delta}}.\nonumber
		\end{align}
		The above inequality holds for probability $1-3\delta$.
	\end{proof}
	
	Thus by union bound, $\Pr\{F_2\}\le 3\delta$. In later proof, we will always assume that $F_1,F_2$ and $F_3$ don't happen. In this case, true mean $\mu_a \in [L_a^t,U_a^t]$ for all rounds $t$.
	Denote $T_1 = 2048H^{(G)}\log (20k|\bA|H^{(G)}\log(2T_1)/\delta)$, then when $t\ge T_1$,
	for all arm $a$ such that $q_a\ge \frac{1}{H^{(G)}\cdot \max\{\Delta_a,\varepsilon/2\}^2}$, note that $ I_a\le \frac{1}{q_a}\le H^{(G)}$, we have 
	\begin{align}
		q_a\ge \frac{1}{H_{m_{\varepsilon, \Delta}^{(G)}}\cdot \max\{\Delta_a,\varepsilon/2\}^2}\ge \frac{3}{t}\log \frac{6|\bA|I_a}{\delta}.\nonumber
	\end{align}
	
	Since $F_1$ doesn't happen, by (\ref{eq_29}),  $|\hat{q}_a-q_a|\le \sqrt{\frac{6q_a}{t}\log \frac{6|\bA|I_a}{\delta}}$ and  
	\begin{align}
		\sqrt{\frac{6q_a}{t}\log \frac{6|\bA|I_a}{\delta}}\le \frac{q_a}{2},\nonumber
	\end{align}
	we have $\hat{q}_a\ge q_a-\sqrt{\frac{6q_a}{t}\log \frac{6|\bA|I_a}{\delta}}\ge \frac{1}{2H^{(G)}\cdot \max\{\Delta_a,\varepsilon/2\}^2}$.

	Hence 
	\begin{align}
		T_a(t)= \hat{q}_a\cdot t\ge \frac{1024}{\max\{\Delta_a,\varepsilon/2\}^2}\log \frac{20kZ_aI_a|\bA|\log(2t)}{\delta}.\nonumber
	\end{align}
	Thus 
	\begin{align}
		\beta_{O}(T_a(t)) &= \sqrt{\frac{64}{T_a(t)}\log \frac{20kZ_aI_a|\bA|\log(2T_a(t))}{\delta}}\nonumber\\
		&\le \frac{\max\{\Delta_a,\varepsilon/2\}}{4}, \nonumber
	\end{align}
	and by Lemma~\ref{lemma:general_obs_confidencebound}, we know the estimation lies in the confidence interval.
	Now we prove the main theorem. The following lemma provides the upper bound of sample complexity
	\begin{lemma}\label{lemma:generalmain}
	    With probability 1-5$\delta$, the algorithm \ref{alg:algorithm3} takes at most $\lceil T\rceil$ rounds such that $T\ge 2112H^{(G)}\log \frac{20H^{(G)}|\bA|\log(2t)}{\delta}$.
	\end{lemma}
	\begin{proof}
	    In the proof we assume $F_1, F_2$ and $F_3$ don't happen. The probability for these events are $1-5\delta$. Assume when $t=\lceil T\rceil$, the algorithm don't terminate at $t$ rounds.
	    
	    Then since $f(x)=\frac{x}{\log (20k|\bA|H^{(G)} \log(2t)/\delta)}$ is a increasing function, $t\ge 2048H^{(G)}\log \frac{20H^{(G)}k|\bA|\log(2t)}{\delta}$ for any $t \in [T_1,T]$. Then from above, for arm $a$ such that $q_a \ge \frac{1}{H^{(G)}\max\{\Delta_a,\varepsilon/2\}^2}$, we have $\beta_a(t)\le \beta_O(T_a(t))\le \frac{\max\{\Delta_a,\varepsilon/2\}}{4}$.
	    Then by Lemma \ref{lemma:twosmall}, at each round at least one intervention will be performed on some arm $a$ with $\beta_I(N_a(t))\ge \beta_a(t)\ge \frac{\max\{\Delta_a,\varepsilon/2\}}{4}$, which implies that $N_a(t)\le \frac{64}{\max\{\Delta_a,\varepsilon/2\}^2}\log \frac{H^{(G)}\log(2t)}{\delta}$. Since these arms are $M$, we have $|M|\le m_{\varepsilon,\Delta}$ and 
	    \begin{align*}
	        T-T_1&\le \sum_{a\in S} \frac{64}{\max\{\Delta_a,\varepsilon/2\}^2}\log \left(\frac{20H^{(G)}k|\bA|\log(2T)}{\delta}\right)\\
	        &\le 64H^{(G)}\log\left(\frac{20H^{(G)}k|\bA|\log(2T)}{\delta}\right).
	    \end{align*}
	    Hence 
	    \begin{align*}
	        T\le T_1+64H^{(G)}\log\left(\frac{20H^{(G)}k|\bA|\log(2T)}{\delta}\right)\le 2112H^{(G)}\log\left(\frac{20H^{(G)}k|\bA|\log(2T)}{\delta}\right),
	    \end{align*}
	    which completes the proof of Lemma. \ref{lemma:generalmain}. \end{proof} 
	    
	    \begin{lemma}\label{lemma:general_transfer}
	        Suppose $T = NQ\log(\frac{20k|\bA|Q\log(2T)}{\delta})$, then $T = O(Q\log(\frac{|\bA|Q}{\delta}))$.
	    \end{lemma}
	    \begin{proof}
	        Similar to Lemma~\ref{lemma:transfer_format}, for $f(x)=x-NQ\log(\frac{20k|\bA|Q\log(2T)}{\delta})$we only need to show that there exists a constant $C$ such that $f(CQ\log\frac{Q|\bA|}{\delta})\ge f(T)=0$.
	        
	        We have 
	        \begin{align*}
	            f(CQ\log\frac{Q|\bA|}{\delta}) &= CQ\log\frac{Q|\bA|}{\delta}-NQ\log (\frac{20k|\bA|Q\log(2CQ\log\frac{Q|\bA|}{\delta})}{\delta})\\
	            &\le CQ\log\frac{Q|\bA|}{\delta}-2NQ\log (\frac{20|\bA|Q\log(2CQ\log\frac{Q|\bA|}{\delta})}{\delta})\\
	            &= (C-2N)Q\log\frac{Q|\bA|}{\delta}-\log(\log 40CQ\log\frac{Q|\bA|}{\delta}))\\
	            &\ge(C-2N)Q\log\frac{Q|\bA|}{\delta}-\log (\log40C) \cdot \log (Q\log\frac{Q|\bA|}{\delta})\\
	            &\ge (C-2N-\log\log40C)Q\log\frac{Q|\bA|}{\delta}.
	        \end{align*}
	        Thus we choose $C$ such that $C-2N-\log\log40C\ge 0$, then we complete the Lemma~\ref{lemma:general_transfer}.
	    \end{proof}
	    By the Lemma~\ref{lemma:general_transfer} above, with probability $1-5\delta$, we have
	    \begin{align*}
	        T=O\left(H^{(G)}\log \left(\frac{|\bA|H^{(G)}}{\delta}\right)\right).
	    \end{align*}
	
	The correctness has been proved in Section \ref{proofthm1}, so we complete the proof of Theorem \ref{theorem:general}.
\end{proof}

\subsection{Proof of Theorem \ref{theorem:lowerbound}}
\begin{proof}
We consider a bandit instance $\xi$ with $\bq$ and probability distribution $P(X_1,X_2,\cdots,X_n,Y)$. Recall $\min_{\bx \in \{0,1\}^n} P(Y=1\mid \bX=\bx)=p_{min}, \max_{\bx \in \{0,1\}^n} P(Y=1\mid \bX=\bx)=p_{max}$ and $p_{max}+2\Delta_{2n+1}+2\varepsilon\le  1$.   For arm $a \in \bA$ with $q_{a}\le \frac{1}{H_{m_{\varepsilon,\Delta}-1}\cdot \max\{\Delta_a,\varepsilon/2\}^2}$, we denote the set of these arms are $M$.  By definition of $m_{\varepsilon,\Delta}$, we know $|M|\ge m_{\varepsilon,\Delta}$. Then for $a=do(X_i=x)\neq \argmin_{a' \in M}\Delta_a'$ (if optimal arm $a^* \in M, a\neq a^*$), we construct bandit instance $\xi'_a$ with probability distribution

$$P'(Y\mid X_1,\cdots,X_n)=\left\{
\begin{aligned}
&P(Y\mid X_1,\cdots,X_n) & X_i\neq x\\
&P(Y\mid X_1,\cdots,X_n)+2(\Delta_a+\varepsilon)& X_i= x\\
\end{aligned}
\right.
$$
Thus for arm $a$ with $q_{a}\le \frac{1}{H_{m_{\varepsilon,\Delta}}^{(P)}\cdot \max\{\Delta_a,\varepsilon/2\}^2}$. Denote $a_{min} = \argmin_{a' \in S}\max\{\Delta_{a'}, \varepsilon/2\}$, (We break the tie arbitrarily),for $a\neq a_{min}$,   $q_{a}\le \frac{1}{2}$.
\begin{align*}
    P'(Y\mid do(X_i=x))&=P'(Y\mid X_i=x)\\&=\sum_{\bx_{-i}}P'(Y\mid X_i=x, \bX_{-i}=\bx_{-i})P'(\bX_{-i}=\bx_{-i})\\&=\sum_{\bx_{-i}}(P(Y\mid X_i=x, \bX_{-i}=\bx_{-i})+2(\Delta_a+\varepsilon))P'(\bX_{-i}=\bx_{-i})\\&=\sum_{\bx_{-i}}(P(Y\mid X_i=x, \bX_{-i}=\bx_{-i})+2(\Delta_a+\varepsilon))P(\bX_{-i}=\bx_{-i})\\&=P(Y\mid X_i=x)+2(\Delta_a+\varepsilon)\\&=P(Y\mid do(X_i=x))+2(\Delta_a+\varepsilon),
\end{align*}
Now we consider other arms $a'=do(X_j=x') \in \bA$, we have 
\begin{align*}
    P'(Y\mid do(X_j=x'))&=P'(Y\mid X_j=x')\\
    &=\sum_{\bx_{-j}}P'(Y\mid X_j=x', \bX_{-j}=\bx_{-j})P'(\bX_{-j}=\bx_{-j})\\
    &=\sum_{\bx_{-j}}P(Y\mid X_j=x', \bX_{-j}=\bx_{-j})P(\bX_{-j}=\bx_{-j})\\&\ \ \ +2(\Delta_a+\varepsilon)\sum_{\bx_{-j,-i}}P(\bX_{-j,-i}=\bx_{-j,-i},X_i=x)\\
    &=P(Y\mid X_j=x')+2(\Delta_a+\varepsilon)\cdot P(X_i=x)\cdot \sum_{\bx_{-j,-i}}P(\bX_{-j,-i}=\bx_{-j,-i})\\
    &=P(Y\mid X_j=x')+2(\Delta_a+\varepsilon)\cdot q_a\cdot \sum_{\bx_{-j,-i}}P(\bX_{-j,-i}=\bx_{-j,-i})\\
    &\le P(Y\mid X_j=x')+(\Delta_a+\varepsilon).\\
\end{align*}
Also, if $a'=do()$, we have 
\begin{align*}
    P'(Y\mid a')=P'(Y)&=\sum_{\bx}P'(Y\mid \bX=\bx)P'(\bX=\bx)\\
    &=\sum_{\bx}P(Y\mid \bX=\bx)P(\bX=\bx)+2(\Delta_a+\varepsilon)\cdot \sum_{\bx_{-i}}P(\bX_{-i}=\bx_{-i}, X_i=x)\\
    &=P(Y)+2(\Delta_a+\varepsilon)\cdot P(X_i=x)\cdot \sum_{\bx_{-i}}P(\bX_{-i}=\bx_{-i})\\
    &\le P(Y)+(\Delta_a+\varepsilon).
\end{align*}
Thus for all $a' \in A$, we have 
\begin{align*}
P'(Y\mid a')&\le P(Y\mid a')+\Delta_a+\varepsilon\\&\le (P(Y\mid a)+\Delta_a+\varepsilon)+\Delta_a+\varepsilon\\&=P(Y\mid a)+2(\Delta_a+\varepsilon)\\&=P'(Y\mid a),
\end{align*}
which means that $a$ is the best arm in bandit environment $\xi_a'$.
Then denote the probability measure for $\xi_a'$ and $\xi$ as $\Pr_a$ and $\Pr$. Denote $Y_t$ and $\bX_t$ as the reward  and observed value at time $t$. 
Define stopping time for the algorithm $\sigma$ with respect to $\mathcal F_t$,  Then from Lemma 19 in \cite{lowerboundproof}, for any event $\zeta \in \mathcal F_{\sigma}$
\begin{equation*}
    \mathbb{E}_{\xi}\left[\sum_{t=1}^\sigma \log \left(\frac{\Pr (Y_t)}{\Pr_{a} (Y_t)}\right)\right]=d(\mbox{Pr}(\zeta), \mbox{Pr}_a (\zeta)),
\end{equation*}
where $d(x,y)=x\log(x/y)+(1-x)\log ((1-x)/(1-y))$ is the binary relatively entropy.

Denote the output of our algorithm is $a^o$. Then since $a\neq a^*$, when we choose $\zeta = \{a^o=a\}$, we have 
$\Pr(\zeta)\le \delta, \Pr_a(\zeta)\ge 1-\delta$ and $d(\Pr(\zeta), \Pr_a(\zeta))\ge \log(\frac{1}{2.4\delta})$

Now note that 
\begin{align}
    &\ \ \ \ \ \mathbb{E}_{\xi}\left[\sum_{t=1}^\sigma \log \left(\frac{\Pr (Y_t)}{\Pr_{a} (Y_t)}\right)\right]\\ &=\sum_{t=1}^{\sigma}\mathbb{E}_{\xi}\left[\log\left(\frac{\Pr (Y_t)}{\Pr_a(Y_t)}\right)\right]\nonumber\\
    &=\sum_{t=1}^{\sigma}\Pr((X_t)_i=x)\left(\sum_{y \in \{0,1\}}\Pr(Y_t=y\mid (X_t)_i=x)\log \frac{\Pr(Y_t=y\mid (X_t)_i=x)}{\Pr_a(Y_t=y\mid (X_t)_i=x)}\right)\label{eq:KLeq}
\end{align}

Denote $B=\Pr(Y_t=y\mid (X_t)_i=x) \in [p_{min},p_{max}]$
\begin{align*}
   &\ \ \ \sum_{y \in \{0,1\}}\Pr(Y_t=y\mid (X_t)_i=x)\log \frac{\Pr(Y_t=y\mid (X_t)_i=x)}{\Pr_a(Y_t=y\mid (X_t)_i=x)} \\
   &= B\log \frac{B}{B+2(\Delta_a+\varepsilon)}+(1-B)\frac{1-B}{1-B-2(\Delta_a+\varepsilon)}\\
   &\le \frac{-2B(\Delta_a+\varepsilon)}{B+2\Delta_a}+\frac{2(1-B)(\Delta_a+\varepsilon)}{1-B-2(\Delta_a+\varepsilon)}\\
   &\le \frac{4(\Delta_a+\varepsilon)^2}{(B+2(\Delta_a+\varepsilon))(1-B-2(\Delta_a+\varepsilon))}\\
   &\le \frac{4(\Delta_a+\varepsilon)^2}{0.9\cdot 0.1},
\end{align*}

Then ~\eqref{eq:KLeq} becomes
\begin{align*}
    \log\left(\frac{1}{2.4\delta}\right)\le \mathbb{E}_{\xi}\left[\sum_{t=1}^\sigma \log \left(\frac{\Pr (Y_t)}{\Pr_{a} (Y_t)}\right)\right]&\le \frac{4(\Delta_a+\varepsilon)^2}{0.09}\sum_{t=1}^{\sigma}\Pr((X_t)_i=x)\\
    &\le \frac{16\max\{\Delta_a,\varepsilon/2\}^2}{0.09}\mathbb{E}_{\xi}[T_{a}(\sigma)],
\end{align*}
where $T_a(\sigma)$ for $a=do(X_i=x)$ means the number of times that $X_i=x$.
Suppose the sample complexity is $T$ for $\xi$, denote $N_a(\sigma)$ be the number of times that $A_t=a$.  We have 
\begin{align*}
    \mathbb{E}_{\xi}[N_a(\sigma)+q_a\cdot \sigma]\ge \mathbb{E}_{\xi}[T_a(\sigma)]\ge \frac{0.09}{16\max\{\Delta_a,\varepsilon/2\}^2}\log \left(\frac{1}{2.4\delta}\right).
\end{align*}
By summing over all $a=do(X_i=x) \in M, a\neq a_{min}$, we get
\begin{align}
   0.09\sum_{\substack{a\in M\setminus\{do()\}\\ a\neq a_{min}}}\frac{1}{16\max\{\Delta_a,\varepsilon/2\}^2}\log\left(\frac{1}{2.4\delta}\right)&\le\sum_{\substack{a\in M\setminus\{do()\}\\ a\neq a_{min}}}\mathbb{E}_{\xi}[N_a(\sigma)+q_a\cdot \sigma]\\&\le \mathbb{E}_{\xi}\left[\sigma+\sum_{\substack{a\in M\setminus\{do()\}\\ a\neq a_{min}}}q_a\cdot \sigma\right] \label{eq:allarmsum}
   \\&\le \mathbb{E}_{\xi}[\sigma]\left(1+\sum_{\substack{a\in M\setminus\{do()\}\\ a\neq a_{min}}}\frac{1}{\max\{\Delta_a,\varepsilon/2\}^2H_{m_{\varepsilon,\Delta}-1}}\right).
\end{align}

Denote $$Q = \sum_{\substack{a\in M\setminus\{do()\}\\ a\neq a_{min}}}\frac{1}{\max\{\Delta_a,\varepsilon/2\}^2}\ge H_{m_{\varepsilon,\Delta}-1}-\min_{i<m_{\varepsilon,\Delta}} \frac{1}{\max\{\Delta_{a_i},\varepsilon/2\}^2}-\frac{1}{\max\{\Delta_{do()}, \varepsilon/2\}^2},$$
then \begin{align*}
    \mathbb{E}_{\xi}[\sigma]&\ge \frac{0.09Q\log\left(\frac{1}{2.4\delta}\right)}{1+Q/H_{m_{\varepsilon,\Delta}-1}}\\&\ge \frac{0.09}{2} \left(H_{m_{\varepsilon,\Delta}-1}-\min_{i<m_{\varepsilon,\Delta}}\frac{1}{\max\{\Delta_{a_i},\varepsilon/2\}^2}-\frac{1}{\max\{\Delta_{do(),\varepsilon/2}\}^2}\right)\log\left(\frac{1}{2.4\delta}\right).
    \end{align*}

\end{proof}

\section{Some Proofs of Lemma}
\subsection{Proof of Lemma~\ref{lemma:relation_m_gap}}\label{sec:proof_lemma1}
\begin{proof}
    By definition, we only need to show that $|\{a\mid q_a\cdot \max\{\Delta_a,\varepsilon/2\}^2<1/H_{2m}\}|\le 2m$. Assume it does not hold, then
       $q_{a_{i}}\cdot \max\{\Delta_{a_{i}},\varepsilon/2\}^2< \frac{1}{H_{2m}}$ for $i=1,2,\cdots,2m+1$. Then for $i\ge m+1,m+2,\cdots,2m+1$, we have 
       \begin{align*}
           q_{a_i}<\frac{1}{H_{2m}\cdot \max\{\Delta_{a_i},\varepsilon/2\}^2}< \frac{1}{\sum_{j=1}^m \frac{1}{\max\{\Delta_{a_j},\varepsilon/2\}^2}\cdot \max\{\Delta_{a_i},\varepsilon/2\}^2}\le \frac{1}{m}.
       \end{align*}
    The inequality above implies the $|\{a\mid q_a<\frac{1}{m}\}|\ge m+1$, which leads to a contradiction for the definition of $m$.
\end{proof}
\subsection{The existence for Admissible Sequence in Graphs without Hidden Variables}
\label{section: existence_adm_seq}
In graph without hidden variables, the admissible-sequence is important for identifying the causal effect. Now we provide an algorithm to show how to find the admissible-sequence in this condition.

\begin{theorem}
    For causal graph $G = (\bX\cup\{Y\}, E)$ without hidden variables, for a set $\bS = \{X_1,\cdots,X_k\}$ and $X_1 \succeq X_2\succeq \cdots \succeq X_k$, the admissible-sequence with respect to $\bS$ and $Y$ can be found by
    \begin{align*}
        \bZ_i = \Pa(X_i)\setminus (\bZ_1\cup \cdots \cup \bZ_{i-1}\cup X_1 \cup \cdots \cup X_{i-1}).
    \end{align*}
    \label{theorem: adm_seq_nohidden}

\end{theorem}
\begin{proof}
    The proof is straightforward. First, $\bZ_i\subseteq \Pa(X_i)$ consists of nondescendants of $\{X_i,X_{i+1},\cdots,X_k\}$ by topological order.
    Second, we need to prove \begin{align}(Y\perp \!\!\! \perp X_i\mid X_1,\cdots,X_{i-1}, \bZ_1,\cdots,\bZ_i)_{G_{\underline{X}_i, \overline{X}_{i+1},\cdots,\overline{X}_{k}}}.\label{eq:independence}\end{align}
    We know that the $\Pa(X_i) \subseteq X_1\cup X_2\cdots X_{i-1}\cup \bZ_1\cup \cdots \bZ_i$. Then it blocks all the backdoor path from $X_i$ to $Y$. Also, since $X_1\cup X_2\cdots X_{i-1}\cup \bZ_1\cup \bZ_2\cdots \bZ_i$ consists of nondescendants of $X_i$, it cannot block any forward path from $X_i$ to $Y$. Also, for any forward path with colliders, namely, $X_i\to \cdots \to X'\leftarrow X''\cdots Y$, the $X'$ cannot be conditioned since it is a descendant for $X_i$. So conditioning on $X_1\cup X_2\cdots X_{i-1}\cup \bZ_1\cup \bZ_2\cdots \bZ_i$ will not active any extra forward path. Hence, there is only original forward path from $X_i$ to $Y$, which means that \eqref{eq:independence} holds.
\end{proof}
\subsection{Proof of Lemma~\ref{lemma:bound_estimate} }
\label{sec:proof_lemma_bound_estimate}

\boundestimate*
\begin{proof}
    Note that our BGLM model is equivalent to a threshold model: For each node $X$, we randomly sample a threshold $\Gamma_X \in [0,1]$, and if $f_X(\theta_X^T \Pa(X))+e_X\ge \Gamma_X$, we let $X = 1$, which means it is activated. At timestep $1$, the $X_1$ is activated, then at timestep $i\ge 2$, $X_i$ is either activated (set it to 1) or deactivated (set it to 0). Then, the BGLM is equivalent to the propagating process above if we uniformly sample $\Gamma_X$ for each node $X$, i.e. $\Gamma_X \sim \mathcal{U}[0,1]$.
    Now we only need to show 
    
    \begin{align}
		|\sigma({\btheta},a)-\sigma(\btheta',a)|\le \E_{\be,\Gamma}\left[\sum_{X \in N_{\bS,Y}}|\bV_{X}^{\top}({\btheta}_{X}-\btheta'_{X})|M^{(1)}\right],
	\end{align}
    Firstly, we have 
    \begin{align*}
        |\sigma(\btheta, a)-\sigma(\btheta',a)|=\mathbb{E}_{\be, \Gamma}\left[\mathbb{I}\{Y \mbox{is activated on $\btheta$}\neq \mathbb{I}\{Y\mbox{is activated on $\btheta'$}\}\}\right],
    \end{align*}
    and we define the event $\mathcal{E}_0^{\be}(X)$ as 
    \begin{align*}
    \mathcal{E}_0^{\be}(X)=\left\{\Gamma|\I \{X\text{ is activated under }\Gamma,\be,\btheta\}\neq \I\{X\text{ is activated under }\Gamma,\be,\btheta'\}\right\}.
\end{align*}

Hence

\begin{align*}
    \left|\sigma(\btheta, a)-\sigma(\btheta',a)\right|&\leq \mathbb{E}_{{\be}}\left[ \Pr_{\Gamma\sim(\mathcal{U}[0,1])^n}\{\mathcal{E}_0^{\be}(Y)\}\right].
\end{align*}

Now since only nodes in $N_{\bS, Y}$ will influence $Y$, we can only consider node $X$ in $N_{\bS,Y}$.

Let $\phi^{\be}(\btheta, \Gamma) = (\phi^{\be}_0(\btheta, \Gamma) \subseteq S, \phi^{\be}_1(\btheta, \Gamma),\cdots ,\phi^{\be}_n(\btheta, \Gamma))$ be the sequence of activated sets on $\btheta$, $0$-mean noise ${\be}$ and threshold factor $\Gamma$. More specifically, $\phi_i(\btheta, \Gamma)$ is the set of nodes activated by time step $i$. For every node $X\in N_{\bS, Y}$, we define the event that $X$ is the first node that has different activation under $\btheta$ and $\btheta'$ as below:

\begin{small}\begin{align*}
    \mathcal{E}_1^{\be}(X)=\{\Gamma|\exists i\in[n],\forall i'<i,\phi^{\be}_{i'}(\btheta,\Gamma)=\phi^{\be}_{i'}(\btheta',\Gamma),X\in(\phi^{\be}_i(\btheta,\Gamma)\backslash\phi^{\be}_i(\btheta',\Gamma)\cup (\phi_{i}^{\be}(\btheta',\Gamma)\backslash\phi^{\be}_i(\btheta,\Gamma)))\}.
\end{align*}\end{small}

Then we have $\mathcal{E}_0^{\be}(Y) \subseteq \cup_{X \in N_{\bS, Y}}\mathcal{E}_1^{\be}(X)$.
We also define other events:
\begin{align*}
    &\mathcal{E}_{2,0}^{\be}(X,i)=\{\Gamma|\forall i'<i,\phi_{i'}^{\be}(\btheta,\Gamma)=\phi_{i'}^{\be}(\btheta',\Gamma),X\not\in\phi_{i-1}^{\be}(\btheta,\Gamma)\},\\
    &\mathcal{E}_{2,1}^{\be}(X,i)=\{\Gamma|\forall i'<i,\phi_{i'}^{\be}(\btheta,\Gamma)=\phi_{i'}^{\be}(\btheta',\Gamma),X\in\phi_i^{\be}(\btheta,\Gamma)\backslash\phi_i^{\be}(\btheta',\Gamma)\},\\
    &\mathcal{E}_{2,2}^{\be}(X,i)=\{\Gamma|\forall i'<i,\phi_{i'}^{\be}(\btheta,\Gamma)=\phi_{i'}^{\be}(\btheta',\Gamma),X\in\phi_i^{\be}(\btheta',\Gamma)\backslash\phi_i^{\be}(\btheta,\Gamma)\},\\
    &\mathcal{E}_{3,1}^{\be}(X,i)=\{\Gamma|X\in\phi_i^{\be}(\btheta,\Gamma)\backslash\phi_i^{\be}(\btheta',\Gamma)\}\\
    &\mathcal{E}_{3,2}^{\be}(X,i)=\{\Gamma|X\in\phi_i^{\be}(\btheta',\Gamma)\backslash\phi_i^{\be}(\btheta,\Gamma)\}.
\end{align*}
Then since $\mathcal{E}_{2,1}^{\be}(X,i)$ and $\mathcal{E}_{2,2}^{\be}(X,i)$ are exclusive, we have 
\begin{align*}
    \Pr_{\Gamma}\{\mathcal{E}^{\be}_1(X)\}=\sum_{i=1}^n\Pr_{\Gamma}\{\mathcal{E}^{\be}_{2,1}(X,i)\}+\sum_{i=1}^n \Pr_{\Gamma}\{\mathcal{E}^{\be}_{2,2}(X,i)\}.
\end{align*}
Now we need to bound the two terms above. First, consider 
$\Pr_{\Gamma}\{\mathcal{E}^{\be}_{2,1}(X,i)\}$, we set $\Gamma_{-X}$ is the vector with all value $\Gamma_{X'}$ of node $X'\neq X$, then we also define the corresponding sub-event $\mathcal{E}^{\be}_{2,1}(X,i,\Gamma_{-X}) \subset \mathcal{E}^{\be}_{2,1}(X,i)$ as the event with value $\Gamma_{-X}$. Define $\mathcal{E}^{\be}_{2,0}(X,i,\Gamma_{-X}) \subset \mathcal{E}^{\be}_{2,0}(X,i)$, $\mathcal{E}^{\be}_{3,1}(X,i,\Gamma_{-X}) \subset \mathcal{E}^{\be}_{3,1}(X,i)$,
$\mathcal{E}^{\be}_{3,2}(X,i,\Gamma_{-X}) \subset \mathcal{E}^{\be}_{3,2}(X,i)$ in a similar way. 

From definition,  $\mathcal{E}^{\be}_{2,1}(X,i,\Gamma_{-X})=\mathcal{E}^{\be}_{3,1}(X,i,\Gamma_{-X})\cup\mathcal{E}^{\be}_{2,0}(X,i,\Gamma_{-X})$, then we have 
\begin{align*}
    \Pr_{\Gamma}\{\mathcal{E}^{\be}_{2,1}(X,i,\Gamma_{-X})\}=\Pr_{\Gamma}\{\mathcal{E}^{\be}_{2,0}(X,i)\}\cdot \Pr_{\Gamma}\{\mathcal{E}^{\be}_{3,1}(X,i,\Gamma_{-X})|\mathcal{E}^{\be}_{2,0}(X,i,\Gamma_{-X})\}.
\end{align*}
Thus, by the definition of BGLM, in $\mathcal{E}^{\be}_{2,0}(X,i,\Gamma_{-X})$, the value of $\Gamma_X$ must lie in an interval with highest value 1. Denote it as $[ \mathcal{W}_{2,0}^{\be}(X,i,\Gamma_{-X}), 1]$, then 
\begin{align*}
    \Pr_{\Gamma_X \sim \mathcal{U}[0,1]}\mathcal{E}^{\be}_{2,0}(X)=1-\mathcal{W}_{2,0}^{\be}(X,i,\Gamma_{-X}).
\end{align*}

Now we consider \begin{align*}\Pr_{\Gamma_X}\{\mathcal{E}^{\be}_{3,1}(X,i,\Gamma_{-X})|\mathcal{E}^{\be}_{2,0}(X,i,\Gamma_{-X})\}.\end{align*}
We first assume $\mathcal{W}_{2,0}^{\be}(X,i,\Gamma_{-X})<1$, otherwise our statement holds trivially.
Then we denote that the nodes activated at timestep $t$ under $\mathcal{E}^{\be}_{2,0}(X,i,\Gamma_{-X})$ as $\phi^{\be}_{t}(\mathcal{E}^{\be}_{2,0}(X,i,\Gamma_{-X}))$.
If the conditional event above holds, we have 

    \begin{scriptsize}\begin{align*}
    f_X\left(\sum_{X'\in \phi^{\be}_{i-1}(\mathcal{E}^{\be}_{2,0}(X,i,\Gamma_{-X}))\cap N(X)}\btheta_{X',X}\right)+e_X<\Gamma_X 
    \leq f_X\left(\sum_{X'\in \phi^{\be}_{i-1}(\mathcal{E}^{\be}_{2,0}(X,i,\Gamma_{-X}))\cap N(X)}\btheta'_{X',X}\right)+e_X,
\end{align*}\end{scriptsize}
or 
\begin{scriptsize}\begin{align*}
    f_X\left(\sum_{X'\in \phi^{\be}_{i-1}(\mathcal{E}^{\be}_{2,0}(X,i,\Gamma_{-X}))\cap N(X)}\btheta_{X',X}\right)+e_X\ge\Gamma_X 
    > f_X\left(\sum_{X'\in \phi^{\be}_{i-1}(\mathcal{E}^{\be}_{2,0}(X,i,\Gamma_{-X}))\cap N(X)}\btheta'_{X',X}\right)+e_X,
\end{align*}\end{scriptsize}
where $\theta_{X',X}$ is the element corresponding to $X'$ in $\theta_X$.

Thus, 
\begin{small}\begin{align*}
    &\Pr_{\Gamma_X \sim \mathcal{U}[0,1]}\{\mathcal{E}^{\be}_{3,1}(X,i,\Gamma_{-X})\cup \mathcal{E}^{\be}_{3,2}(X,i,\Gamma_{-X})|\mathcal{E}^{\be}_{2,0}(X,i,\Gamma_{-X})\}\\
    &\ \ =\frac{|f_X\left(\sum_{X'\in \phi^{\be}_{i-1}(\mathcal{E}^{\be}_{2,0}(X,i,\Gamma_{-X}))\cap N(X)}\btheta_{X',X}\right)-f_X\left(\sum_{X'\in \phi^{\be}_{i-1}(\mathcal{E}^{\be}_{2,0}(X,i,\Gamma_{-X}))\cap N(X)}\btheta'_{X',X}\right)|}{1-\mathcal{W}_{2,0}^{\be}(X,i,\Gamma_{-X})}
\end{align*}\end{small}

Thus we have 

\begin{align*}
    &\ \ \  \Pr_{\Gamma}\{\mathcal{E}^{\be}_{2,1}(X,i,\Gamma_{-X})\cup\mathcal{E}^{\be}_{2,2}(X,i,\Gamma_{-X}) \}\\& =\Pr_{\Gamma}\{\mathcal{E}^{\be}_{2,0}(X,i)\}\cdot \Pr_{\Gamma}\{\mathcal{E}^{\be}_{3,1}(X,i,\Gamma_{-X})\cup \mathcal{E}^{\be}_{3,2}(X,i,\Gamma_{-X})|\mathcal{E}^{\be}_{2,0}(X,i,\Gamma_{-X})\}\\
    &=\left|f_X\left(\sum_{X'\in \phi^{\be}_{i-1}(\mathcal{E}^{\be}_{2,0}(X,i,\Gamma_{-X}))\cap N(X)}\btheta_{X',X}\right)-f_X\left(\sum_{X'\in \phi^{\be}_{i-1}(\mathcal{E}^{\be}_{2,0}(X,i,\Gamma_{-X}))\cap N(X)}\btheta'_{X',X}\right)\right|\\
    &\le M^{(1)}\left|\left(\sum_{X'\in \phi^{\be}_{i-1}(\mathcal{E}^{\be}_{2,0}(X,i,\Gamma_{-X}))\cap N(X)}\btheta_{X',X}\right)-\left(\sum_{X'\in \phi^{\be}_{i-1}(\mathcal{E}^{\be}_{2,0}(X,i,\Gamma_{-X}))\cap N(X)}\btheta'_{X',X}\right)\right|.
\end{align*}

When $\mathcal{E}^{\be}_{2,0}(X) = \emptyset$, both two sides are zero, so it holds in general.

Now we define $\mathcal{E}^{\be}_{4,0}(X,i,\Gamma_{-X})=\{\Gamma\mid \Gamma = (\Gamma_X, \Gamma_{-X})\mid X \notin \phi_{i-1}^{\be}(\btheta, \Gamma)\}$, then $\mathcal{E}^{\be}_{2,0}(X,i,\Gamma_{-X})\subseteq \mathcal{E}^{\be}_{4,0}(X,i,\Gamma_{-X})$.
In addition, when $\mathcal{E}^{\be}_{2,0}(X,i,\Gamma_{-X})\neq \emptyset,$ $\phi_{i'}^{\be}(\mathcal{E}^{\be}_{2,0}(X,i,\Gamma_{-X}))=\phi_{i'}^{\be}(\mathcal{E}^{\be}_{4,0}(X,i,\Gamma_{-X}))$ for all $i'<i$. Thus we have 

\begin{align*}
    &\ \ \ \Pr_{\Gamma}\{\mathcal{E}^{\be}_{2,1}(X,i,\Gamma_{-X})\cup\mathcal{E}^{\be}_{2,2}(X,i,\Gamma_{-X}) \}\\
    &\le M^{(1)}\left|\left(\sum_{X'\in \phi^{\be}_{i-1}(\mathcal{E}^{\be}_{4,0}(X,i,\Gamma_{-X}))\cap N(X)}\btheta_{X',X}\right)-\left(\sum_{X'\in \phi^{\be}_{i-1}(\mathcal{E}^{\be}_{4,0}(X,i,\Gamma_{-X}))\cap N(X)}\btheta'_{X',X}\right)\right|.
\end{align*}
Now we can get 
\begin{align*}
    \Pr_{\Gamma}\{\mathcal{E}_1^{\be}(X)\}&=\int_{\Gamma_{-X}}\sum_{i=1}^n\Pr_{\gamma_X\sim\mathcal{U}[0,1]}\{\mathcal{E}^{\be}_{2,1}(X,i,\Gamma_{-X})
    \cup\mathcal{E}^{\be}_{2,2}(X,i,\Gamma_{-X})\}\text{d}\Gamma_{-X}\\
    &=\int_{\Gamma_{-X}}\Pr_{\gamma_X\sim\mathcal{U}[0,1]}\{\mathcal{E}^{\be}_{2,1}(X,i^*,\Gamma_{-X})
    \cup\mathcal{E}^{\be}_{2,2}(X,i^*,\Gamma_{-X})\}\text{d}\Gamma_{-X}\\
    &\leq \int_{\Gamma_{-X}}\left|\sum_{X'\in \phi^{\be}_{i^*-1}(\mathcal{E}^{\be}_{4,0}(X,i^*,\Gamma_{-X}))\cap N(X)}
    (\theta_{X',X}-\theta'_{X',X})\right|M^{(1)}\text{d}\Gamma_{-X}\\
    &=\mathbb{E}_{\Gamma_{-X}}\left[\left|\sum_{X'\in \phi^{\be}_{i^*-1}(\mathcal{E}^{\be}_{4,0}(X,i^*,\Gamma_{-X}))\cap N(X)}
    (\theta_{X',X}-\theta'_{X',X})\right|\right]M^{(1)}\\
    &=  \mathbb{E}_{\Gamma_{-X}}\left[|\bV_X(\theta_X-\theta'_X)|M^{(1)}\right],
\end{align*}
where $i^*$ is the topological order of $X$ in graph $G$, and the second inequality is because $\mathcal{E}_{2,1}^{\be}(X,i,\Gamma_{-X}) \neq \emptyset$ only when $i = i^*$.
Summing over all node $X \in N_{\bS,Y}$, we complete the proof.

\end{proof}
\subsection{Proof of Lemma~\ref{lemma:eigenvalue_estimate}}
\eigenvalueestimate*
\begin{proof}
    The all proof is very similar to the proof in \cite{CCB2022}, for the completeness, we provide them here. 
    Note that $\hat{\btheta}_{t,X}$ satisfies $\nabla L_{t,X}(\hat{\btheta}_{t,X})=0$, where \begin{align*}
    \nabla L_{t,X}(\btheta_X)=\sum_{i=1}^t[X^t-f_X(\bV_{i,X}^T\btheta_X)]\bV_{i,X}.
\end{align*}

Define $G(\btheta_X)=\sum_{i=1}^t(f_X(\bV_{i,X}^T\btheta_X)-f_X(\bV_{i,X}^T\btheta_X^*))\bV_{i,X}$. Thus $G(\btheta_X^*)=0$ and $G(\hat{\btheta}_{t,X})=\sum_{i=1}^t\varepsilon'_{i,X}\bV_{i,X}$, where $\varepsilon_{i,X}'=X^i-f_X(\bV_{i,X}^T\btheta_X^*)$. Now note that $\mathbb{E}[\varepsilon'_{i,X}|\bV_{i,X}]=0$ and $\varepsilon'_{i,X}=X^i-f_X(\bV_{i,X}^T\btheta^*_X)\in[-1,1]$, then $\varepsilon'_{i,X}$ is 1-subgaussian. Let $Z=G(\hat{\btheta}_{t,X})=\sum_{i=1}^t\varepsilon'_{i,X}\bV_{i,X}$

\paragraph{Step 1: Consistency of $\hat{\btheta}_{t,X}$}

For any $\btheta_1,\btheta_2\in\mathbb{R}^{|\Pa(X)|}$,  $\exists\bar{\btheta}=s\btheta_1+(1-s)\btheta_2,0<s<1$ such that \begin{align*}
    G(\btheta_1)-G(\btheta_2)&=\left[\sum_{i=1}^t\dot{f}_X(\bV_{i,X}^T\bar{\btheta})\bV_{i,X}\bV_{i,X}^T\right](\btheta_1-\btheta_2)\\
    &\triangleq F(\bar{\btheta})(\btheta_1-\btheta_2).
\end{align*}
Since $f$ is strictly increasing, $\dot{f}>0$, then $G(\btheta)$ is an injection and $G^{-1}$ is well-defined.

Now let $\mathcal{B}_{\eta } = \{\btheta\mid ||\btheta-\btheta^*||\le \eta\}$, then define $\kappa_\eta := \inf_{\btheta \in \mathcal{B}_\eta, X\neq 0}\dot{f}(X^T\btheta)>0$. The following lemma helps our proof, and it can be found in Lemma A of \cite{1999Chen}:
\begin{lemma}[\cite{1999Chen}]
    $\{\btheta\mid ||G(\btheta)||_{M_{t,X}^{-1}}\le \kappa_\eta \eta \sqrt{\lambda_{min}(M_{t,X})}\}\subseteq \mathcal{B}_\eta.$
\end{lemma}
The next lemma provides an upper bound of $||Z||_{M_{t,X}^{-1}}$:
\begin{lemma}[\cite{zhang2022online}]\label{lemma:upperbound||Z||}
    For any $\delta>0$, the event $\mathcal{E}_G:= \{||Z||_{M_{t,X}^{-1}}\le 4\sqrt{|\Pa(X)|+\ln(1/\delta)}\}$ holds with probability at least $1-\delta$.
\end{lemma}
By the above two lemmas, when $\mathcal{E}_G$ holds, for any $\eta\ge \frac{4}{\kappa_\eta}\sqrt{\frac{|\Pa(X)|+\ln(1/\delta)}{\lambda_{\min}(M_{t,X})}}$, we have $||\hat{\btheta}_{t,X}-\btheta^*||\le \eta.$ Choose $\eta = 1$, we know $1\ge \frac{4}{\kappa}\sqrt{\frac{|\Pa(X)|+\ln(1/\delta)}{\lambda_{\min}(M_{t,X})}}$, then with probability $1-\delta$ $||\hat{\btheta}_{t,X}-\btheta^*||\le 1$.

\paragraph{Step 2: Normality of $\hat{\btheta}_{t,X}.$}
Now we assume $||\hat{\btheta}_{t,X}-\btheta^*||\le 1$ holds. Define $\Delta = \hat{\btheta}_{t,X}-\btheta^*$, then $\exists s\in[0,1]$ such that $Z=G(\hat{\btheta}_{t,X})-G(\btheta_X^*)=(H+E)\Delta$, where $\bar{\btheta}=s\btheta^*_X+(1-s)\hat{\btheta}_{t,X}$, $H=F(\btheta^*_X)=\sum_{i=1}^t\dot{f}_X(\bV_{i,X}^T\btheta^*_X)\bV_{i,X}\bV_{i,X}^T$ and $E=F(\bar{\btheta})-F(\btheta^*_X)$. Then, according to mean value theorem, we have 
\begin{align*}
    E&=\sum_{i=1}^t(\dot{f}_X(\bV_{i,X}\cdot\overline{\btheta})-\dot{f}_X(\bV_{i,X}\cdot \btheta^*_X))\bV_{i,X}\bV_{i,X}^T\\
    &=\sum_{i=1}^t\ddot{f}_X(r_i)\bV_{i,X}^T\Delta \bV_{i,X}\bV_{i,X}^T\\
    &\le \sum_{i=1}^tM^{(2)}\bV_{i,X}^T\Delta \bV_{i,X}\bV_{i,X}^T
\end{align*}
for some $r_i\in\mathbb{R}$. Thus we have 
\begin{align*}
    \bv^TH^{-1/2}EH^{-1/2}\bv
    &\leq\sum_{i=1}^tL_{f_X}^{(2)}||\bV_{i,X}||||\Delta||||\bv^T H^{-1/2}\bV_{i,X}||^2\\
    &\leq M^{(2)}\sqrt{|\Pa(X)|}||\Delta||(\bv^T H^{-1/2}(\sum_{i=1}^t\bV_{i,X}\bV_{i,X}^T)H^{-1/2}\bv)\\
    &\leq\frac{M^{(2)}\sqrt{|\Pa(X)|}}{\kappa}||\Delta||||\bv||^2,
\end{align*}
hence we know 
\begin{align*}
    ||H^{-1/2}EH^{-1/2}||
    &\leq\frac{M^{(2)}\sqrt{|\Pa(X)|}}{\kappa}||\Delta||\\
    &\leq\frac{4M^{(2)}\sqrt{|\Pa(X)|}}{\kappa^2}\sqrt{\frac{|\Pa(X)|+\ln\frac{1}{\delta}}{\lambda_{\min}(M_{t,X})}}\\
    &\le \frac{1}{2},
\end{align*}
where the last inequality is because \begin{small}\begin{align*}\lambda_{\min}(M_{t,X})\geq 512\frac{(M^{(2)})^2}{\kappa^4}|\Pa(X)|\left(|\Pa(X)|+\ln\frac{1}{\delta}\right)>64\frac{(M^{(2)})^2}{\kappa^4}|\Pa(X)|\left(|\Pa(X)|+\ln\frac{1}{\delta}\right).\end{align*}\end{small}

Now for any $\bv \in \mathbb{R}^{|\Pa(X)|}$, we have \begin{align*}
    \bv^T(\hat{\btheta}_{t,X}-\btheta^*_X)&=\bv^T(H+E)^{-1}Z\\
    &=\bv^T H^{-1}Z-\bv^T H^{-1}E(H+E)^{-1}Z.
\end{align*}
The second equality is correct from $H+E = F(\bar{\btheta})\succeq \kappa M_{t,X}\succeq 0$.

Define $D\triangleq(\bV_{1,X},\bV_{2,X},\cdots,\bV_{t,X})^T\in\mathbb{R}^{t\times|\Pa(X)|}$. Then $D^T D=\sum_{i=1}^t\bV_{i,X}\bV_{i,X}^T=M_{t,X}$. By the Hoeffding's inequality \cite{hoeffding1994probability}, \begin{align*}
    \Pr(\left|\bv^T H^{-1}Z\geq a\right|)&\leq\exp\left(-\frac{a^2}{2||\bv^T H^{-1}D^T||^2}\right)\\
    &=\exp\left(-\frac{a^2}{2\bv^T H^{-1}D^T DH^{-1}\bv}\right)\\
    &\leq \exp\left(-\frac{a^2\kappa^2}{2||\bv||^2_{M_{t,X}^{-1}}}\right).
\end{align*} 
The last inequality holds because $H\succeq \kappa M_{t,X}=\kappa D^T D$.

Thus with probability $1-2\delta$, $| \bv^T H^{-1}Z|\leq\frac{\sqrt{2\ln{1/\delta}}}{\kappa}||\bv||_{M_{t,X}^{-1}}$.

For the second term, we know \begin{align}
    |\bv^T H^{-1}E(H+E)^{-1} Z|
    &\leq ||\bv||_{H^{-1}}||H^{-\frac{1}{2}}E(H+E)^{-1}Z||\nonumber\\
    &\leq ||\bv||_{H^{-1}}||H^{-\frac{1}{2}}E(H+E)^{-1}H^{\frac{1}{2}}||\ ||Z||_{H^{-1}}\nonumber\\
    &\leq \frac{1}{\kappa}||\bv||_{M^{-1}_{t,X}}||H^{-\frac{1}{2}}E(H+E)^{-1}H^{\frac{1}{2}}||\ ||Z||_{M^{-1}_{t,X}}.\label{ineq:E(H+E)}
\end{align}
Then we get
\begin{align*}
    ||H^{-\frac{1}{2}}E(H+E)^{-1}H^{\frac{1}{2}}||
    &=||H^{-\frac{1}{2}}E(H^{-1}-H^{-1}E(H+E)^{-1})^{-1}H^{\frac{1}{2}}||\\
    &=||H^{-\frac{1}{2}}EH^{\frac{1}{2}}+H^{-\frac{1}{2}}EH^{-1}E(H+E)^{-1}H^{\frac{1}{2}}||\\
    &\leq ||H^{-\frac{1}{2}}EH^{\frac{1}{2}}||+||H^{-\frac{1}{2}}EH^{-\frac{1}{2}}||\ ||H^{-\frac{1}{2}}E(H+E)^{-1}H^{\frac{1}{2}}||,
\end{align*}
where the first inequality is derived by $(H+E)^{-1}=H^{-1}-H^{-1}E(H+E)^{-1}$.

Then we can get 
\begin{align*}
    ||H^{-\frac{1}{2}}E(H+E)^{-1}H^{\frac{1}{2}}||\le \frac{||H^{-\frac{1}{2}}EH^{\frac{1}{2}}||}{1-||H^{-\frac{1}{2}}EH^{\frac{1}{2}}||}&\le 2||H^{-\frac{1}{2}}EH^{\frac{1}{2}}||\\
    &\le \frac{8M^{(2)}\sqrt{|\Pa(X)|}}{\kappa^2}\sqrt{\frac{|\Pa(X)|+\ln\frac{1}{\delta}}{\lambda_{\min}(M_{t,X})}}
\end{align*}
Thus by \eqref{ineq:E(H+E)} and Lemma~\ref{lemma:upperbound||Z||}
\begin{align*}
    \left|v^T H^{-1}E(H+E)^{-1} Z\right|
    &\leq\frac{32L_{f_X}^{(2)}\sqrt{|\Pa(X)|}(|\Pa(X)|+\log\frac{1}{\delta})}{\kappa^3\sqrt{\lambda_{\min}(M_{t,X})}}||\bv||_{M_{t,X}^{-1}}.
\end{align*}
So we have 
\begin{align*}
    \left|\bv^T(\hat{\btheta}_{t,X}-\btheta^*_X)\right|
    &\leq\left(\frac{32L_{f_X}^{(2)}\sqrt{|\Pa(X)|}(|\Pa(X)|+\log\frac{1}{\delta})}{\kappa^3\sqrt{\lambda_{\min}(M_{t,X})}}
    +\frac{\sqrt{2\ln{1/\delta}}}{\kappa}\right)||\bv||_{M_{t,X}^{-1}}\\
    &\leq\frac{3}{\kappa}\sqrt{\log(1/\delta)}||\bv||_{M_{t,X}^{-1}},
\end{align*}
where the last inequality is because
\begin{align*}
    \lambda_{\min}(M_{t,X})\geq \frac{512|\Pa(X)|(L_{f_X}^{(2)})^2}{\kappa^4}\left(|\Pa(X)|^2+\ln\frac{1}{\delta}\right).
\end{align*}
By replace $\delta$ with $\delta/3nt^2$, we complete the proof.
\end{proof}

\section{Experiments}
\label{sec:experiment}
In this section, we provide some experiments supporting our theoretical result for CCPE-BGLM and CCPE-General. 

\subsection{CCPE-BGLM}
\paragraph{Experiment 1}
First, we provide the experiments for our CCPE-BGLM algorithm. We construct a causal graph with 9 nodes $X_1,\cdots,X_8$ and $X_0$, such that $X_i\succeq X_{i+1}$. Then, we randomly choose two nodes in $X_1,\cdots,X_{i-1}$ and also $X_0$ to be the parent of $X_i(i\ge 1)$. $Y$ has 4 parents, and they are randomly chosen in $\bX = \{X_1\cdots,X_8\}$. For $X_0$, we know $P(X_0=1)=1$. 
For node $X_i$ and their parent $X_i^{(1)}, X_i^{(2)}$, $P(X_i=1)=0.4X_0 + 0.1X_i^{(1)}+0.1X_i^{(2)}$. (If $i=2$, $P(X_i=1)=0.4X_0+0.1X_i^{(1)}=0.4X_0+0.1X_1$; If $i=1$, $P(X_1=1)=0.4X_0.$) Suppose the parents of reward variable are $X^{(1)},X^{(2)},X^{(3)},X^{(4)}$   The reward variable is defined by $P(Y=1)=0.3X^{(1)}+0.3X^{(2)}+0.3X^{(3)}+0.05X^{(4)}$. The action set is $\{do(\bS=\textbf{1}\mid |\bS|=3, \bS \subset \bX)\}$. Hence the optimal arm is $do(\{X^{(1)}, X^{(2)}, X^{(3)}\}=\textbf{1})$. 

We choose 4 algorithms in this experiment: LUCB in \cite{LUCB2012}, lilUCB-heuristic in \cite{lilUCB}, Propagating-Inference in \cite{icml2018-propagatinginference} and our CCPE-BGLM. LUCB and lilUCB-heuristic are classical pure exploration algorithm. Because in previous causal bandit literature, Propagating-Inference is the only algorithm considering combinatorial action set without prior knowledge $P(\Pa(Y)\mid a)$ for action $a \in \bA$, we choose it in this experiment. Note that the criteria of Propagating-Inference algorithm is simple regret, hence it cannot directly compare to our pure exploration algorithm. We choose to compare the error probability at some fixed time $T$ instead. 
In this criteria, Propagating-Inference algorithm will have an extra knowledge of budget $T$ while LUCB, lilUCB-heuristic and CCPE-BGLM not. 
To implement the Propagating Inference algorithm, we follows the modification in \cite{icml2018-propagatinginference} to make this algorithm more efficient and accurate by setting $\lambda = 0$ and $\eta_A = 1/C$. (Defined and stated in \cite{icml2018-propagatinginference}.) 
For CCPE-BGLM, we ignore the condition that $t\ge \max\{\frac{cD}{\eta^2}\log \frac{nt^2}{\delta}, \frac{1024(M^{(2)})^2(4D^2-3)D}{\kappa^4\eta}(D^2+\ln\frac{3nt^2}{\delta})\}$, to make it more efficient. During our experiment, the error probability is smaller than other algorithm even if we ignore this condition. 
Also, to make this algorithm more efficient, we update observational confidence bound (Line 11) each 50 rounds. (This will not influence the proof of Theorem~\ref{theorem:linear}.)
For LUCB, lilUCB-heuristic and CCPE-BGLM, we find the best exploration parameter $\alpha$, $\alpha_I$ and $\alpha_O$ by grid search from $\{0.05,0.1,\cdots,1\}$. (Exploration parameter $\alpha$ for UCB-type algorithm  is a constant multiplied in front of the confidence radius, which should be tuned in practice. e.g.(\cite{linucb}, \cite{findingalleps-good}.) For this task, we find $\alpha = 0.3$, $\alpha_O = 0.05, \alpha_I = 0.4$. We choose $T =50+50i$ for $0\le i\le 9$. For each time $T$, we run 100 iterations and average the result.  

As the Figure~\ref{fig:linear_experiments} shows, even if our algorithm does not know the budget $T$, our algorithm converges quicker than all other algorithms. 
\begin{figure}[htb] 
	\centering 
	\includegraphics[width=0.50\textwidth]{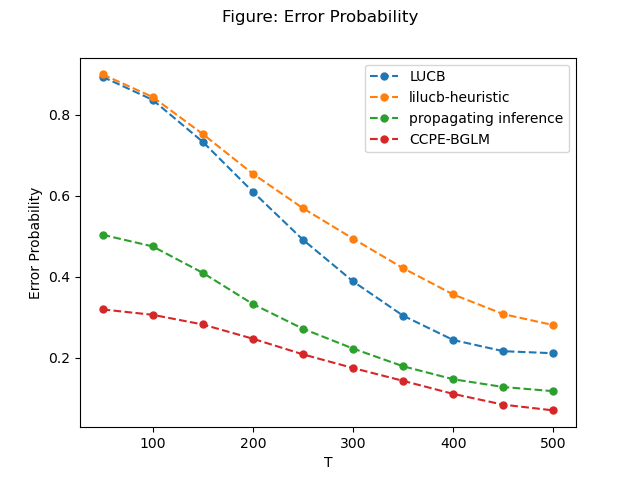}
	\caption{Error Probability for Experiment 1}
	\label{fig:linear_experiments} 
\end{figure}

\subsection{CCPE-General}
In this subsection, we provide the experiments for CCPE-General algorithm. We also choose 4 algorithms, LUCB, lilUCB-heuristic, Propagating-Inference and CCPE-General (called "adm\_seq" in figure because it utilizes admissible sequence.). Since the Propagating-Inference cannot hold for general graph with hidden variables, we first compare them in graphs without hidden variables. Then we also compare LUCB, lilUCB-heuristic and our algorithm in the graphs with hidden variables.

\paragraph{Experiment 2}
we construct the graph with 7 nodes $X_1,\cdots,X_7$ such that $X_i\succeq X_{i+1}$. Then, we randomly choose two nodes in $X_1,\cdots,X_{i-1}$ as parents of $X_i$. The reward variable $Y$ has 5 parents $X^{(i)}=X_{i+2}, 1\le i\le 5$. We choose $P(X_1=1)=0.5$, $P(X_2=1)=0.55$ if $X_0 = 1$ and otherwise $P(X_2=1)=0.45$.
For $i\ge 2$, and two parents $X_i^{(1)}, X_i^{(2)}$ of $X_i$, $P(X_i=1)=0.55$ if $X_i^{(1)}=X_{i}^{(2)}$ and otherwise $P(X_i=1)=0.45$.
For reward variable $Y$, 
\begin{align}
    P(Y=1) = \left\{
\begin{array}{lcc}
0.9      &  {X_i^{(1)}=X_i^{(2)}=1} \\
0.7+0.05X_i^{(1)}+0.05X_i^{(2)}     & {X_i^{(3)}=X_i^{(4)}=1}\\
0         &  {\mbox{Otherwise}}
\end{array} \right.
\end{align}

We define the action set $\{do(\bS=\bs)\mid |\bS|=2, \bs \in \{0,1\}^2\}$, then the optimal arm is $do(\{X_i^{(1)}, X_i^{(2)}\}=\textbf{1})$. We choose $\alpha_O = 0.25, \alpha_I = 0.4$, and exploration parameters $\alpha$ for LUCB and lilUCB are both 0.3. For each time $T = 150+50i$ for $0\le i\le 9$, we run 100 times and average the result to get the error probability. The result is shown in Figure~\ref{fig:general_experiment}. We note that our algorithm performs almost the same as Propagating-Inference algorithm. 
Our CCPE-General algorithm is a fixed confidence algorithm without requirement for budget $T$, and our algorithm can be applied to causal graphs with hidden variables. 
\begin{figure}[htb] 
	\centering 
	\includegraphics[width=0.50\textwidth]{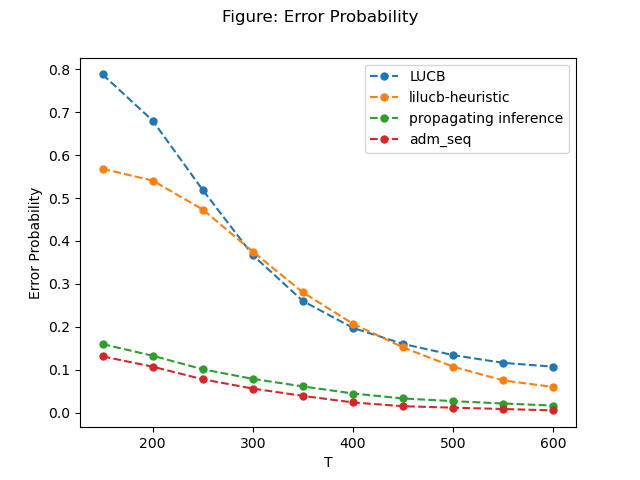}
	\caption{Error Probabiltiy for Experiment 2}
	\label{fig:general_experiment} 
\end{figure}

\paragraph{Experiment 3}
In this paragraph, we provide an experiment to show that CCPE-General algorithm can be applied to broader causal graphs with hidden variables. Since there is no previous algorithm working on both combinatorial action set and existence of hidden variable, we compare our result with LUCB and lilUCB-heuristic.

Our causal graph are constructed as follows: $\bX = X_0,X_1,\cdots,X_{n+1}$, where $X_i\to Y$ and $X_1\to X_i$ for $2\le i\le n+1$, $X_{1}\to X_0$. $\bU = \{U_1,\cdots,U_n\}$ and $U_i\to X_{i+2}, U_i\to X_{0}$ for $2\le i\le n+1$. For action set $\{do(\bS=\bs)\mid |\bS|=2, \bS\subset \{X_2,\cdots,X_{n+1}\}\}$. 
Each $U_i$ satisfies $P(U_i=1)=0.5$. $P(X_0=1)=\min\{\frac{1}{n}\sum_{i=0}^{n-1}U_i+0.1X_1,1\}, $. $P(X_1=1)=0.5$. $P(X_i=1)=0.5$ if $X_1 = U_{i-2}= 1$ and otherwise 0.4. For the reward variable $Y$, $P(Y=1)=0.4X_2+0.4X_3+\frac{0.2}{n}\sum_{i=4}^{n+2}X_i$. 

For this task, by grid search, we set $\alpha=0.25$ for exploration parameter of LUCB and lilUCB, and $\alpha_O=0.3, \alpha_I=0.4$ for CCPE-General algorithm (In the figures below, we call our algorithm "\mbox{adm\_seq}" since it uses admissible sequence.) We compare the error probability and sample complexity for them. The results are shown in Figure~\ref{fig:general_hidden_samplecomp} and Figure~\ref{fig:general_hidden_errorprob}. Our CCPE-General algorithm wins in both metrics.

\begin{figure}[H] 
	 \centering
	 
	 \subfigure[Sample Complexity for Experiment 3]{
    \label{fig:general_hidden_samplecomp} 
   \includegraphics[width=0.4\textwidth]{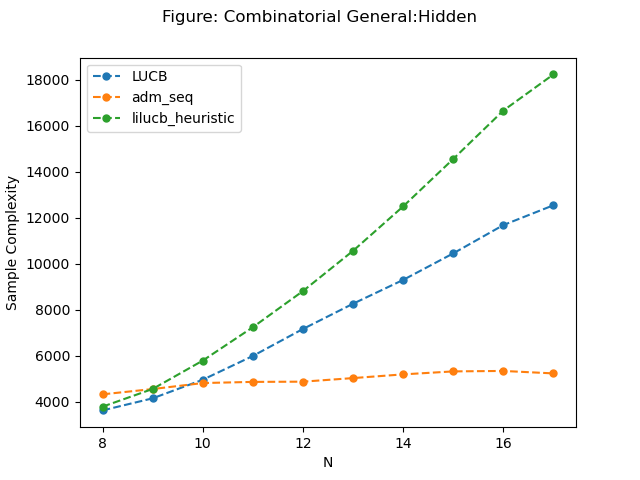}}
    \hspace{0.5in}
	\subfigure[Error Probability for Experiment 3]{
    \label{fig:general_hidden_errorprob} 
    \includegraphics[width=0.4\textwidth]{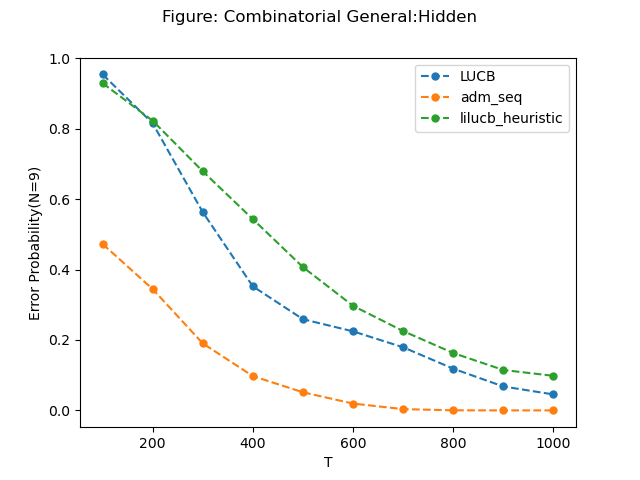}}
  \caption{Experiment 3}
  \label{fig:twopicture} 

\end{figure}

\section{Fixed Budget Causal Bandit Algorithm}
In this section, we provide a preliminary fixed budget causal bandit algorithm, which based on successive reject algorithm and our previous analysis for causal bandit. The previous causal bandit algorithm in fixed budget always directly estimate the observation threshold $m$. However, to derive a gap-dependent result, this method does not work. Our Causal Successive Reject avoids the estimation for observation threshold and get a better gap-dependent result.
Note that $T_a(t) = \min_{\bz}T_{a,\bz}(t)$ for $\bz \in \{0,1\}^{|Z_a|}$, then we can get the simple causal successive reject algorithm as follows:

\begin{algorithm}[H]
	\caption{{\rm Causal Successive Reject}}
	\label{alg:Causal SR}
	\begin{algorithmic}[1]
	\STATE Input: Causal Graph $G$, action set $A$, budget $T$, parameter $\varepsilon$.
	\STATE Initialize $t=1, T_a=0,\hat{\mu}_a=0$ for all arms $a \in \bA$. Define $|\bA|=N$. $A_0 = A$
	
	\STATE Perform $do()$ for $T/2$ times, and update $T_a(t)$ for all $a$.
	
	\STATE Set $n_k = \frac{\frac{T}{2}-N}{\overline{\log N}(N+1-k)}$ for $k=1,2,\cdots,n-1$.
	
	\FOR{each phase $k=1,2,\cdots,n-1$:}

	    \FOR{$i=1,2,\cdots,\lceil(N+1-k)(n_k-n_{k-1})\rceil$}
	    
	    \STATE Perform intervention $a$ for action with least $T_a+N_a$, and $N_a = N_a+1$.
	    \ENDFOR
	    
	    \STATE Denote $a_k=\argmin_{a \in \bA_{k-1}}\hat{\mu}_{a}$, where $\hat{\mu}_a$ follows the same definition of Algorithm~\ref{theorem:general}.
	    $A_k = A_{k-1}\backslash\{a_k\}$. 
	    
	    \IF{$|A_k| = 1$}
	    \RETURN $A_k$.
	    \ENDIF

	\ENDFOR

\end{algorithmic}
\end{algorithm}

\begin{theorem}\label{thm:fixed_budget_parallel}
    The algorithm 4 will return the $\varepsilon$-optimal arm within error probability
    \begin{align*}
        P(\mu_{a^o}< \mu_{a^*}-\varepsilon)\le 4I_aN^2\exp\left\{-\frac{\frac{T}{2}-N}{128\overline{\log N}H_3}\right\},
    \end{align*}
    where $H_3 = \max_{k=1,2,\cdots,N}\{\alpha_{k}^{-1}(\max\{\Delta^{(k)}, \varepsilon\})^{-2}\}$, and $\alpha_k$ is defined by 
    $$ \alpha_{k}=\left\{
    \begin{aligned}
    \frac{1+\sum_{i=k+1}^N \frac{1}{i}}{m} \ \ \  & \mbox{if} \ \ k>m\\
    \frac{1}{k} + \frac{\sum_{i=m+1}^N \frac{1}{i}}{m}\ \ \ \  & \mbox{if} \ \ k\le m 
    \end{aligned}
    \right.
    $$
    where $m$ is defined in \ref{def:othresholdm} with respect to $q_a$ similar to Theorem~\ref{theorem:general}
\end{theorem}

To show that our algorithm outperforms the classical successive reject and sequential halving algorithm, it is obvious that $H_3\le H_2$, where $H_2 = \max_{k=1,2,\cdots,N}\{k\max\{\Delta^{(k)}, \varepsilon\}^{-2}\}$, since $\alpha_k\ge \frac{1}{k}$.

\begin{proof}
     We also denote $T_a(t), N_a(t)$ as the value of $T_a$ and $N_a$ at the round $t$.
     The main idea is that: Each stage we spend half budget to observe, and spend the remaining budget to supplement the arms which are not observed enough. The main idea of proof is to show that in each stage, each arm in $A_k$ has $T_a(t)+D_a(t) \ge m_k$ times, which leads to a brand-new result.
     
     Denote the set of arm $S = \{a\in \bA\mid q_a<1/m\}$, then $|S|\le m$.
First,  for $a \notin S$, by chernoff bound, it has been observed by $\hat{q}_a\cdot T/2 \ge \frac{T}{4m}$ with probability $1-\delta$, where $\delta = 6N\cdot \exp\{-\frac{T}{24m}\}$. Hence $T_a(T/2)\ge \frac{T}{4m}$ for $a \notin S$.

First we prove the following lemma:
\begin{lemma}\label{lemma:each_stage_number_times}
    After stage $1\le k\le N-m$, all the arms in $A_k$ must have $N_a(t)+T_a(t)\ge m_k$, where $m_k =\sum_{i=1}^k \frac{(N+1-i)(n_i-n_{i-1})}{2m}\ge \lceil \frac{\frac{T}{2}-N}{2\overline{\log N}\cdot m}(1+\sum_{i=N+2-k}^N\frac{1}{i})\rceil$.
\end{lemma}

\begin{proof}
Let $D_a(t) = T_a(t)+N_a(t)$.
 Denote $m_0 = n_0 = 0$, then 
For $a\notin S$, $T_a(t) \ge \frac{T}{4m}$. Thus number of arm $a$ with $T_a(t)\le \frac{T}{4m}$ is less than $m$. Then the intervention in stage $k$ will only performed on $D_a(t)\le \frac{T}{4m}$ unless all the arms have $D_a(t)\ge \frac{T}{4m}\ge m_k$ times.

If all arms have $D_a(t)\ge \frac{T}{4m}\ge m_k$ times, the lemma holds. If it is not true, the $(N+1-k)(n_k-n_{k-1})$ interventions will performed on at most $m$ arms. Hence all the arms must have $N_a(t)\ge m_{k-1}+ \frac{(N+1-k)(n_k-n_{k-1})}{m}=m_{k-1}+2(m_k-m_{k-1}) \ge m_k$ times after stage $k\le N-m$.

Then $m_k = \sum_{i=1}^k \frac{(N+1-i)(n_i-n_{i-1})}{2m}\ge \sum_{i=1}^{k-1} \frac{n_i}{2m} +\frac{(N+1-k)n_k}{2m}\ge \frac{\frac{T}{2}-N}{\overline{\log N}\cdot 2m}(1+\sum_{i=N+2-k}^N\frac{1}{i})$.
\end{proof}

\begin{lemma}\label{lemma:each_stage_number_times_k_ge_m}
    After stage $k>N-m$, all the arms in $A_k$ have $T_a(t)+N_a(t)\ge m_k'$, where $m_k' = \frac{\frac{T}{2}-N}{2\bar{\log N}}\cdot \alpha_{N+1-k}$. 
\end{lemma}
\begin{proof}
For $a\in A_k$, in stage $k>N-m$, $|A_k|= N-k+1\le m$. All the arms in $A_k$ must have $T_a(t)+N_a(t)\ge m_{k-1}'+(N+1-k)(n_k-n_{k-1})/(N+1-k)= m_{k-1}'+n_k-n_{k-1}$ times, where $m_{N-m}' = m_{N-m}$. 

Thus after stage $k>N-m$, all the arms in $A_k$ must have
\begin{align*}
T_a(t)+N_a(t)&\ge m_{N-m}+\sum_{N-m<i\le k}(n_i-n_{i-1}) \\&= m_{N-m}+(n_k-n_{N-m})\\&\ge \frac{\frac{T}{2}-N}{\overline{\log N}\cdot m}(1+\sum_{i=m+2}^N\frac{1}{i})+1+(\frac{\frac{T}{2}-N}{\overline{\log N}}(\frac{1}{N+1-k}-\frac{1}{m+1}))-1\\
& =  \frac{\frac{T}{2}-N}{\overline{\log N}\cdot m}(\sum_{i=m+1}^N\frac{1}{i})+(\frac{\frac{T}{2}-N}{\overline{\log N}}(\frac{1}{N+1-k}))\\
&= \frac{\frac{T}{2}-N}{\overline{\log N}}(\frac{1}{N+1-k}+\frac{1}{m(\sum_{i=m+1}^N \frac{1}{i})^{-1}})\\
&= \frac{\frac{T}{2}-N}{\overline{\log N}}\cdot \alpha_{N+1-k}\\
&\ge m_k'.
\end{align*}
\end{proof}

\begin{lemma}\label{lemma:fixedbudget_singletimebound}
		In round t, with probability $1-\frac{\delta}{8nt^3}$,
		\begin{align}
			|\hat{\mu}_{obs,a}-\mu_a|<4\sqrt{\frac{1}{T_a(t)}\log\frac{4I_a}{\delta}}.
		\end{align}
	\end{lemma}
	\begin{proof}
	    When $T_a(t)\ge 12\log \frac{4I_a}{\delta}$, we know this lemma is trivial since $\mu_a, \hat{\mu}_{obs,a} \in [0,1]$. Otherwise, if $t<\frac{6}{q_a}$, define $Q=\frac{6}{q_a}\log(\frac{4I_a}{\delta})$, based on $T_a(t)\ge 12\log\frac{4I_a}{\delta}$, then 
		\begin{align*}
		    P\left(t<\frac{6}{q_a}\log(1/\delta)\right)\le P\left(T_a(Q)\ge 12\log\frac{4I_a}{\delta}\right). 
		\end{align*}
		Thus by Chernoff bound, we know 
		\begin{align*}
		    P\left(T_a(Q)\ge 12\log\frac{4I_a}{\delta}\right)= P\left(\hat{q}_a(Q)\ge 2q_a\right)\le \delta,
		\end{align*}
		where $\hat{q}_a(Q) = \frac{T_a(Q)}{Q}$.

		Hence with probability at least $1-\delta$, now we have $t\ge \frac{6}{q_a}\log(4I_a/\delta)$.
		Also, since $\hat{P}(\bZ_i=z_i,X_i=x_i, i\le l-1)=T_{a,\bz,l}(t)/t$, by Chernoff bound, when $t\ge \frac{6}{q_a}\log (4I_a/\delta)$, with probability
		$1-\exp\{-\frac{P(\bZ_i=z_i,X_i=x_i, i\le l-1)\cdot t}{3}\}\ge 1-\delta$, we have
		\begin{align*}
		    \hat{P}(\bZ_i=z_i,X_i=x_i, i\le l-1)&\le 2P(\bZ_i=z_i,X_i=x_i, i\le l-1).
		\end{align*}
	    
		 Now since
		 
		 \begin{align*}
		     P(\bZ_l = \bz_l\mid \bZ_i = \bz_i, X_i = x_i, i\le l-1)&\ge \frac{ q_a}{P(\bZ_i=z_i, X_i=x_i, i\le l-1)}\\
		     &\ge \frac{ q_a}{2\hat{P}(\bZ_i=z_i, X_i=x_i, i\le l-1)}\\
		     &\ge \frac{q_at}{2T_{a,\bz,l}(t)}\\
		     &\ge \frac{3}{T_{a,\bz,l}(t)}\log \frac{4I_a}{\delta}.
		 \end{align*}

		By Hoeffding's inequality and Chernoff bound, for $a=do(X=x)$, 
		\begin{align*}
			|r_{a,\bz}(t)-P(Y=1\mid \bS=\bs, \bz=\bz)|&\le\sqrt{\frac{1}{2T_{a,\bz}(t)}\log\frac{4I_a}{\delta}},
		\end{align*}
		\begin{align*}
			&|p_{a,\bz,l}(t)-P(\bZ_l=\bz_l\mid \bZ_i=\bz_i, X_i=x_i,i\le l-1)|\\ &\ \ \ \ \ \ \ \  \le \sqrt{\frac{3P(\bZ_l=\bz_l\mid \bZ_i=\bz_i, X_i=x_i,i\le l-1)}{t}\log \frac{4I_a}{\delta}}
		\end{align*}
		at round $2t$ with probability $1-2Z\frac{\delta}{2I_a}=1-\delta$. Hence by Lemma~\ref{lemma:T_a,z,l(t)bound}, Eq \eqref{eq:halfprobability}, we get
		\begin{align}
			&\ \ \ \ \ \hat{\mu}_{obs,a}\\&= \sum_{\bz}r_{a,\bz}(t)\cdot \hat{P}_{a,\bz,k}(t)\nonumber\\
			&\le \sum_{\bz}r_{a,\bz}(t)\cdot \hat{P}_{a,\bz,k-1}(t)\cdot P_t(\bZ_k=\bz_k\mid \bZ_i=\bz_i, X_{i}=x_i, i\le k-1)\nonumber\\&\ \ \ \ \ +\frac{1}{2}\sum_{\bz}r_{a,\bz}(t)\hat{P}_{a,\bz,k-1}(t)\sqrt{\frac{3P(\bZ_k = \bz_k\mid \bZ_i=\bz_i, X_{i}=x_i, i\le k-1)}{2^{|\bZ_k|+1}T_{a,\bz}(t)}\log \frac{4I_a}{\delta}}\nonumber\\
			& \le \sum_{\bz}r_{a,\bz}(t)\cdot \hat{P}_{a,\bz,k-1}(t)\cdot P_t(\bZ_k=\bz_k\mid \bZ_i=\bz_i, X_{i}=x_i, i\le k-1)\nonumber\\&\ \ \ \ +\frac{1}{2}\sum_{\bz}\hat{P}_{a,\bz,k-1}(t)\cdot \sqrt{\frac{3P(\bZ_k = \bz_k\mid \bZ_i=\bz_i, X_{i}=x_i, i\le k-1))}{2^{|\bZ_k|+1}T_{a,\bz}(t)}\log \frac{4I_a}{\delta}}\nonumber\\
			&\le \sum_{\bz}r_{a,\bz}(t)\cdot \hat{P}_{a,\bz,k-1}(t)\cdot P_t(\bZ_k=\bz_k\mid \bZ_i=\bz_i, X_{i}=x_i, i\le k-1)\nonumber\\&\ \ \ \ +\frac{1}{2}\sum_{\bz_k}\cdot \sqrt{\frac{3P(\bZ_k = \bz_k\mid \bZ_i=\bz_i, X_{i}=x_i, i\le k-1))}{2^{|\bZ_k|+1}T_{a,\bz}(t)}\log \frac{4I_a}{\delta}}\nonumber\\
			&\le \sum_{\bz}r_{a,\bz}(t)\cdot \hat{P}_{a,\bz,k-1}(t)\cdot P_t(\bZ_k=\bz_k\mid \bZ_i=\bz_i, X_{i}=x_i, i\le k-1)\nonumber\\
			&\ \ \ \ + \frac{1}{2}\sqrt{\frac{3\cdot 2^{|\bZ_k|}}{2\cdot 2^{|\bZ_k|}T_{a,\bz}(t)}\log \frac{4I_a}{\delta}}  \ \ \ \ \ (\mbox{Cauchy-Schwarz Inequality})\nonumber\\
			&\le \sum_{\bz}r_{a,\bz}(t)\cdot \hat{P}_{a,\bz,k-1}(t)\cdot P_t(\bZ_k=\bz_k\mid \bZ_i=\bz_i, X_{i}=x_i, i\le k-1)\\&\ \ \ \ \ + \frac{1}{2}\sqrt{\frac{3}{2\cdot T_{a,\bz}(t)}\log \frac{4I_a}{\delta}}\nonumber\\
			&\le \cdots \nonumber\\
			&\le \sum_{\bz}r_{a,\bz}(t)P_{a,\bz,k}+\frac{1}{2}\sum_{i=1}^k \sqrt{\frac{3}{2^{i}T_{a,\bz}(t)}\log \frac{4I_a}{\delta}}\nonumber\\
			&\le \mu_a+\frac{1}{2-\sqrt{2}}\sqrt{\frac{3}{T_{a,\bz}(t)}\log \frac{4I_a}{\delta}} + \sqrt{\frac{1}{2T_{a,\bz}(t)}\log\frac{2}{\delta}}.\nonumber\\
			&\le \mu_a+ 4\sqrt{\frac{1}{T_a(t)}\log\frac{4I_a}{\delta}}.\nonumber
		\end{align}
	\end{proof}
Now we prove another lemma to bound the error probability of each stage.

\begin{lemma}
    For an arm $a \in \bA$, $N_a(t)+T_a(t) = D_a(t)$, then we have 
    \begin{align}
        P(|\hat{\mu}_a-\mu_a|>\epsilon)<4I_a\exp\{-D_a(t)(\epsilon^2/32\}.
    \end{align}
\end{lemma}
\begin{proof}
    We know $N_a(t)\ge \frac{D_a(t)}{2}$ or $T_a(t)\ge \frac{N_a(t)}{2}$. When $N_a(t)\ge \frac{N_a(t)}{2}$, by Hoeffding's inequality, we know that
    \begin{align*}
        P(|\hat{\mu}_a-\mu_a|>\epsilon)<2\exp\{-2N_a(t)(\epsilon/2)^2\}<2\exp\{-D_a(t)(\epsilon/2)^2\}.
    \end{align*}

    When $T_a(t)\ge \frac{D_a(t)}{2}$, by Lemma~\ref{lemma:fixedbudget_singletimebound}, we know  \begin{align*}
        P(|\hat{\mu}_a-\mu_a|>\epsilon)<4I_a\exp\{-N_a(t)\epsilon^2/16\}<4I_a\exp\{-D_a(t)\epsilon^2/32\}.
    \end{align*}
    
    Then we complete the proof.
\end{proof}

Hence the event that 
\begin{align*}
    \zeta = \left\{\forall i \in \{1,2,\cdots,N\}, \forall a \in \bA_i, |\hat{\mu}_{a}-\mu_a| < \frac{1}{2}\max\{\Delta^{(N+1-i)},\varepsilon\}\right\}
\end{align*}
doesn't happen within probability at most 
\begin{align*}
    &\ \ \ \ \sum_{i=1}^n \sum_{a \in \bA_i} 4I_a\exp\{-m_i(\frac{\max\{(\Delta^{(N+1-i)})^2, \varepsilon\}}{2})^2/32\}\\&\le \sum_{i=1}^n \sum_{a \in \bA_i} 4I_a\exp\left\{-\alpha_{N+1-k}\cdot \max\{(\Delta^{(N+1-i)}), \varepsilon\}^2\frac{\frac{T}{2}-N}{128\overline{\log N}}\right\}\\
    &\le 4I_aN^2\exp\left\{-\frac{\frac{T}{2}-N}{128\overline{\log N}H_3}\right\},
\end{align*}
where $H_3 = \max_{i=1,2,\cdots,n}\{\alpha_{i}^{-1}(\max\{\Delta^{(i)},\varepsilon\})^{-2}\}$. 

Now we prove that under event $\zeta$, the algorithm output a $\varepsilon$-optimal arm. 

For each stage $k$, we prove that one of the following condition will be satisfied:

(1). All arms in $A_{k-1}$ are $\varepsilon-$optimal.

(2). Stage $k$ eliminate an non-optimal arm $a_k\neq a^*$

In fact, assume (1) does not hold, then there exists at least one arm which is not $\varepsilon-$optimal. Since $|A_{k-1}| = N+1-k$, there must exist an arm $a \in \bA_{k-1}$ with $\mu_{a^*}-\mu_{a}\ge \max\{\varepsilon, \Delta^{(N+1-k)}\}$. Hence because of event $\zeta$,  after stage $k$, all arms in $A_k$ satisfy $\hat{\mu}_{a}-\mu_a< \frac{1}{2}\max\{\Delta^{(N+1-k)}, \varepsilon\}$. Hence 

\begin{align*}
    \hat{\mu}_{a}&\le \mu_a + \frac{1}{2}\max\{\Delta^{(N+1-k)}, \varepsilon\}\\&\le \mu_{a^*}-\max\{\Delta^{(N+1-k)},\varepsilon\}+\frac{1}{2}\max\{\Delta^{(N+1-k)}, \varepsilon\}\\&
    < \hat{\mu}_{a^*}+\max\{\Delta^{(N+1-k)}, \varepsilon\}-\max\{\Delta^{(N+1-k)},\varepsilon\}\\&=\hat{\mu}_{a^*}.
\end{align*}
So the optimal arm $a^*$ will not be eliminated.

Hence if (2) always happen, the remaining arm will be the optimal arm. Otherwise, if (1) happens, the algorithm will return an $\varepsilon$-optimal arm. Hence we complete the proof.
\end{proof}


}
\end{document}